\documentclass{article} % For LaTeX2e
\usepackage{iclr2027_conference,times}
\usepackage{microtype}

\usepackage{amsmath,amsfonts,bm}

\def\eqref#1{equation~\ref{#1}}
\def\1{\bm{1}}

\DeclareMathAlphabet{\mathsfit}{\encodingdefault}{\sfdefault}{m}{sl}
\SetMathAlphabet{\mathsfit}{bold}{\encodingdefault}{\sfdefault}{bx}{n}

\usepackage[breaklinks=true]{hyperref}
\usepackage{url}
\usepackage{graphicx}
\usepackage{booktabs}
\usepackage{amsmath}
\usepackage{amssymb}
\usepackage{amsthm}
\newtheorem{theorem}{Theorem}
\newtheorem*{theorem*}{Theorem}
\usepackage{subcaption}
\usepackage{multirow}
\usepackage{array}
\usepackage{float}
\usepackage{wrapfig}
\usepackage{algorithm}
\usepackage{algorithmic}

\title{DADP: Dynamic Activity-Dependent Pruning, A Reverse Hebbian-Inspired Structural Pruning Method}

\author{Bhushan Deshpande \\
Independent Researcher, Pune, India \\
\texttt{bhushan1729dsp@gmail.com}
}

\iclrfinalcopy % Uncomment for camera-ready version, but NOT for submission.
\begin{document}

\maketitle

\begin{abstract}
Modern neural networks are heavily over-parameterized. This redundancy incurs substantial compute and memory overhead during training and inference. Existing pruning methods rely on post-hoc magnitude thresholds or static initialization heuristics. Consequently, they often require manual per-layer sparsity targets or expensive retraining cycles. We propose \textbf{Dynamic Activity-Dependent Pruning (DADP)}, a biologically inspired structural plasticity mechanism. During training, DADP measures connection importance via the accumulated product of pre-synaptic activations and post-synaptic error gradients. Using a single global threshold instead of fixed layer budgets, DADP dynamically allocates sparsity across network depth while naturally inducing neuron- and channel-level pruning. Across MLP, VGG-16, ResNet-18, BiLSTM-CRF, and MiniBERT architectures, DADP matches or outperforms Magnitude, SNIP and RigL, retaining \textbf{73.67\%} accuracy (dense baseline: 76.06\%) at \textbf{99\%} sparsity on ResNet-18. Finally, matrix-based Shannon entropy and effective rank measurements confirm that DADP preserves latent feature diversity at extreme sparsities without representation collapse.

\noindent\textbf{Code:} \url{https://github.com/bhushan1729/hebbian_learning}
\end{abstract}

\section{Introduction}

As deep neural network architectures scale in size, their representation power grows. However, this capacity expansion demands massive compute and memory during training and inference \citep{hoffmann2022training}. Network pruning has thus become essential both for efficient deployment \citep{lecun1989optimal, han2015learning} and for understanding generalization \citep{arora2018stronger} and biological computation \citep{tanaka2019from}.

Existing pruning techniques generally fall into three categories. Post-hoc methods train a full model first, prune low-importance weights, and fine-tune the rest. This requires expensive prune-retrain cycles \citep{han2016deep}. Static methods prune connections once at initialisation before training. But their accuracy remains strictly bounded by initial weight sensitivity \citep{lee2019snip, tanaka2020pruning}. Dynamic sparse training methods prune and regrow weights during training \citep{evci2020rigging}. But they force rigid, pre-specified sparsity targets onto each layer, preventing the model from finding its natural capacity.

In contrast to artificial networks, biological brains achieve extraordinary computational efficiency through dynamic, highly sparse connectivity. A key driver of biological learning is structural plasticity via activity-dependent synapse elimination \citep{hebb1949organization, han2024developmental}. Neurobiological studies show that active pathways are preserved while inactive connections are eliminated \citep{faust2021mechanisms, mikuni2013arc, yasuda2021activity}. During early development, the brain grows excess synapses and prunes those that stay idle or fail to support functional circuits \citep{huttenlocher1979synaptic, huttenlocher2002neural}. Inspired by this process, we propose \textbf{Dynamic Activity-Dependent Pruning (DADP)}, a structural plasticity algorithm grounded in a reverse Hebbian mechanism. Rather than evaluating weights by magnitude alone, DADP continuously measures connection importance during training as the expected product of pre-synaptic activation and post-synaptic error gradient: $I_{ij} = \mathbb{E}\left[ \left| a_i \cdot \frac{\partial L}{\partial y_j} \right| \right]$. A connection survives only if its pre-synaptic neuron fires ($a_i \neq 0$) and its post-synaptic target is actively reducing error ($\frac{\partial L}{\partial y_j} \neq 0$).

Our primary contributions are summarized as follows:
\begin{itemize}
    \item \textbf{Reverse Hebbian Pruning Metric:} We introduce a single-pass structural plasticity criterion ($I_{ij} = \mathbb{E}[|a_i \cdot \frac{\partial L}{\partial y_j}|]$) that scores connection importance using both forward activation and backpropagated error sensitivity.
    \item \textbf{Self-Regulating Capacity Allocation:} Using a single global threshold $\tau$ instead of manual layer budgets, DADP establishes a self-correcting negative feedback loop (a prune-recovery cycle) that dynamically allocates layer-wise capacity as training progresses.
    \item \textbf{Emergent Structural Properties:} We show that element-wise activity pruning naturally yields structured channel and neuron elimination, preserves skip connections in residual blocks, and maintains unimodal, zero-centered weight distributions unlike the bimodal cutoffs of magnitude pruning.
    \item \textbf{Empirical and Manifold Evaluation:} Across vision (VGG-16, ResNet-18) and sequence models (BiLSTM-CRF, MiniBERT), DADP matches or outperforms methods like SNIP and RigL at up to $99\%$ sparsity. Using matrix-based Shannon entropy and effective rank, we confirm that DADP preserves latent feature diversity without representation collapse.
    \item \textbf{Hyperparameter and Initialization Robustness:} We demonstrate that DADP remains stable across a $50\times$ range of pruning intervals ($\Delta t \in [100, 5000]$) and converges consistently across diverse variance-scaling initializations (Kaiming, Xavier, Orthogonal) without requiring threshold retuning.
\end{itemize}

\section{Related Work}
\textbf{Classical Pruning and Saliency.} Early pruning works largely fall into penalty-based methods adding $L_0$ or $L_1$ regularizers \citep{chauvin1989backpropagation, weigend1991generalization} and saliency-based methods estimating parameter importance via loss sensitivity \citep{reed1993pruning}. Classical saliency metrics include local weight or neuron sensitivity \citep{mozer1989using, karnin1990simple}, second-order Taylor expansions of the loss function, and diagonal or full Hessian matrices (e.g., Optimal Brain Damage \citep{lecun1989optimal} and Optimal Brain Surgeon \citep{hassibi1993second}) to remove parameters and update surviving weights. While theoretically grounded, computing and inverting Hessians quickly becomes computationally unaffordable for modern deep networks.

\textbf{Static, Dynamic, and Post-Hoc Sparsity.} To avoid second-order costs, magnitude pruning became the standard post-hoc baseline \citep{han2015learning}, laying empirical groundwork for the Lottery Ticket Hypothesis \citep{frankle2019lottery}. To eliminate expensive train-prune-retrain cycles, static methods prune once at initialization based on gradient sensitivity (SNIP) \citep{lee2019snip}, gradient flow preservation (GraSP) \citep{wang2020grasp}, or synaptic flow conservation (SynFlow) \citep{tanaka2020pruning}. Dynamic sparse training algorithms update connectivity during optimization (SET \citep{mocanu2018scalable}; DeepR \citep{bellec2018deep}; RigL \citep{evci2020rigging}). However, these techniques typically require predefined, rigid layer-wise sparsity schedules, preventing the network from reallocating capacity across depths.

\textbf{Structured and Activity-Dependent Pruning.} Structured pruning removes whole channels or attention heads to achieve hardware speedups without sparse kernels \citep{sze2017efficient}, using dependency graphs to handle layer coupling \citep{fang2023depgraph}. In large language models, post-training methods like SparseGPT \citep{frantar2023sparsegpt} and Wanda \citep{sun2024wanda} leverage weight reconstruction or activation magnitudes for calibration-based compression. Closer to biological mechanisms, activity-dependent heuristics evaluate connection utility via local signal activity \citep{dekhovich2023nnrelief, bingham2025fine}. Unlike methods reliant on static pre-training metrics, post-hoc fine-tuning, or manual layer targets, DADP uses a single global reverse-Hebbian threshold to dynamically distribute layer sparsity during standard single-pass training.

\section{Methodology}

In this section, we describe the formulation of \textbf{Dynamic Activity-Dependent Pruning (DADP)}, explain its biological intuition and mathematical basis, detail how unstructured pruning naturally causes structured neuron and channel death, and outline our experimental setup and evaluation metrics.

\subsection{Synaptic Importance Metric \& Reverse Hebbian Intuition}

DADP evaluates the structural utility of a connection (synapse) $w_{ij}^l$ connecting pre-synaptic neuron $i$ in layer $l-1$ to post-synaptic neuron $j$ in layer $l$ using an expected activation-derivative product:
\begin{equation}
I_{ij}^l = \mathbb{E} \left[ \left| a_i^{l-1} \cdot \frac{\partial \mathcal{L}}{\partial y_j^l} \right| \right]
\end{equation}
where $a_i^{l-1}$ is the pre-synaptic activation from layer $l-1$, $y_j^l$ is the pre-activation sum of neuron $j$ at layer $l$ before non-linear activation, $\mathcal{L}$ is the task loss, and $\mathbb{E}[\cdot]$ denotes the expectation over a temporal training window $\Delta T$.

By the multivariable chain rule, importance metric $I_{ij}^l$ measures the expected weight gradient magnitude accumulated over a temporal window $\Delta T$ (e.g., $100$--$500$ steps) to filter single-batch noise. While classical Hebbian learning preserves co-activating connections (``neurons that fire together, wire together''), DADP implements a reverse principle, eliminating connections if either the pre-synaptic neuron is silent ($a_i^{l-1} \approx 0$) or the post-synaptic target is insensitive to error ($\frac{\partial \mathcal{L}}{\partial y_j^l} \approx 0$).

\begin{theorem}[First-Order Taylor Loss Sensitivity]
\label{thm:taylor}
The activation-derivative metric $I_{ij}^l = \mathbb{E}[|a_i^{l-1} \cdot \frac{\partial \mathcal{L}}{\partial y_j^l}|]$ is equivalent to the magnitude of the first-order Taylor expansion of loss change $\Delta \mathcal{L}$ resulting from setting weight $w_{ij}^l$ to zero, normalized per unit weight magnitude. (Proof in Appendix~\ref{sec:appendix_proofs}).
\end{theorem}

\subsection{Dynamic Mask Updates, Self-Correction, \& Emergent Structured Death}

DADP maintains a binary mask matrix $M^l \in \{0, 1\}^{|W^l|}$ for each layer $l$, initialized to all ones. The masked forward pass is given by $y^l = (W^l \odot M^l) a^{l-1} + b^l$. At step boundaries where $t \equiv 0 \pmod{\Delta T}$, masks are updated using a single global threshold $\tau$:
\begin{equation}
M_{ij}^l(t) = M_{ij}^l(t - \Delta T) \cdot \mathbb{I}\left( I_{ij}^l(t) \ge \tau \right)
\end{equation}
Pruning in DADP is permanent. Once a connection's importance score falls below $\tau$, its entry in the binary mask is set to zero ($M_{ij}^l = 0$) and it is permanently eliminated from the computation graph without regrowth. The complete training and pruning procedure is summarized in Algorithm~\ref{alg:dadp} (Appendix~\ref{sec:appendix_algorithm}).

\subsubsection{Emergent Structured Death \& Physical Compression}

While DADP prunes individual connections ($w_{ij}^l$), neuron and filter deletion emerge naturally during training. For a neuron to remain active, it must have at least one active incoming path and at least one active outgoing path. If all incoming connections to neuron $j$ drop to zero, it receives no forward activation ($a_j^l = 0$). Conversely, if all outgoing connections are pruned to zero, no error gradient flows back to neuron $j$ during backpropagation. On the next pruning interval, importance scores for all its incoming connections collapse to zero ($I_{ij}^l = 0$), automatically pruning these connections. If a neuron loses its active incoming or outgoing path, it is dead and we remove such neurons. Similarly, when all connections in a convolutional filter drop to zero ($\sum_{j,h,w} M_{ijhw}^l = 0$), the filter is dead. We physically delete the dead filter slice alongside its Batch Normalization parameters, and contract the corresponding input channels in the next layer.

\subsection{Experimental Setup}

We evaluate DADP across five architectures: Multi-Layer Perceptrons (\textbf{MLP}), \textbf{VGG-16} \citep{simonyan2014very}, \textbf{ResNet-18} \citep{he2016deep}, \textbf{BiLSTM-CRF} \citep{lample2016neural}, and \textbf{MiniBERT} \citep{devlin2019bert} on four benchmark datasets: \textbf{MNIST} \citep{lecun1998mnist}, \textbf{CIFAR-10}$\citep{krizhevsky2009learning}$, \textbf{CoNLL-2003}$\citep{tjongkimsang2003introduction}$, and \textbf{SST-2} $\citep{socher2013recursive}$. Detailed layer specifications are provided in Appendix~\ref{sec:appendix_architectures}. We compare DADP against four primary baselines: \textbf{Dense Baseline}, \textbf{Magnitude Pruning} \citep{han2015learning}, \textbf{SNIP} \citep{lee2019snip}, and \textbf{RigL} \citep{evci2020rigging}. Specific implementation details and training schedule adaptations for these baseline methods are provided in Appendix~\ref{sec:appendix_baselines}. All models are trained using the Adam optimizer \citep{kingma2014adam} with a batch size of $64$ and learning rates of $\eta = 1\text{e-}3$ (MLP, VGG-16, ResNet-18, BiLSTM-CRF) and $\eta = 2\text{e-}5$ (MiniBERT). Models are trained for $20$ epochs using Cross-Entropy loss (or CRF loss for sequence labeling). For DADP, default pruning intervals are set to $\Delta T = 500$ steps across all architectures, with global absolute thresholds swept across $\tau \in [1\text{e-}6, 1\text{e-}3]$. DADP statistical sweeps are evaluated across $3$ random seeds to report mean $\pm$ standard deviation, whereas baseline methods are evaluated across a single seed due to compute constraints. All experiments were executed on free NVIDIA T4 GPU instances on Google Colab and Kaggle platforms. Additionally, we conduct ablation studies evaluating sensitivity to weight initialization schemes (Kaiming \citep{he2015delving}, Xavier \citep{glorot2010understanding}, Orthogonal \citep{saxe2013exact}) and pruning intervals ($\Delta T \in [100, 5000]$ steps).

\subsection{Information-Theoretic Representation Metrics}

Motivated by recent work on layer-wise representation quality analysis \citep{skean2025layer}, we evaluate internal feature quality without relying on output label projections. To the best of our knowledge, we are the first to perform layer-wise matrix-based entropy and effective rank analysis on dynamically pruned sparse networks.

Specifically, for an intermediate layer activation matrix $Z \in \mathbb{R}^{N \times D}$, where $N$ denotes the batch size (number of samples) and $D$ is the feature dimension, we compute the centralized Gram Matrix $K = (Z - \bar{Z})(Z - \bar{Z})^\top \in \mathbb{R}^{N \times N}$, where $\bar{Z}$ is the mean activation matrix across the batch. The normalized eigenvalues $\{\lambda_i\}_{i=1}^N$ of $K$ form a probability distribution $p_i = \lambda_i / \sum_{j=1}^N \lambda_j$, representing the fraction of total feature variance captured along each principal component axis. We evaluate two core metrics:
\begin{enumerate}
    \item \textbf{Matrix-Based Normalized Shannon Entropy ($S_1$)}: Quantifies latent feature diversity and information compression across layer depths:
    \begin{equation}
    S_{\text{norm}}(Z) = -\frac{1}{\log_2(N)} \sum_{i=1}^N p_i \log_2(p_i + \epsilon)
    \end{equation}
    \item \textbf{Effective Rank ($\text{EffRank}$)}: Quantifies the effective geometric dimension of the feature space:
    \begin{equation}
    \text{EffRank}(Z) = \exp \left( -\sum_{i=1}^N p_i \ln(p_i + \epsilon) \right)
    \end{equation}
\end{enumerate}
Here, $\log_2(N)$ normalizes the entropy metric between $0$ (complete collapse into a single dimension) and $1$ (maximum uniform feature spread), and $\epsilon > 0$ is a small constant preventing numerical instability.

\section{Empirical Evaluation \& Analysis}

In this section, we present detailed empirical evaluations of \textbf{Dynamic Activity-Dependent Pruning (DADP)}. We benchmark DADP against dense, static and dynamic pruning baselines, evaluate dynamic optimization mechanics, analyze latent representation quality through information theory, and perform thorough sensitivity and ablation studies.

\subsection{Comparative Accuracy-Sparsity Benchmarks}

We evaluate DADP across five architectures (MLP, VGG-16, ResNet-18, BiLSTM-CRF, MiniBERT) on four benchmark datasets (MNIST, CIFAR-10, CoNLL-2003, SST-2). Table~\ref{tab:main_benchmarks} summarizes performance comparisons against Dense, Magnitude Pruning \citep{han2015learning}, SNIP \citep{lee2019snip}, and RigL \citep{evci2020rigging}. Detailed architecture descriptions are provided in Appendix~\ref{sec:appendix_architectures}.

Because DADP determines structural sparsity organically via a single global threshold $\tau$, its final sparsity cannot be manually hardcoded. For fair comparison in Table~\ref{tab:main_benchmarks}, we compare each DADP configuration with baseline methods evaluated at the closest corresponding target sparsity. The complete, un-truncated benchmark evaluations across all thresholds and sparsities are provided in Appendix~\ref{sec:appendix_full_tables}.

\begin{table}[t]
\caption{Performance comparison of DADP against baseline pruning methods at different target sparsities. DADP values report mean $\pm$ standard deviation across 3 random seeds ($42, 512, 1729$) from statistical sweeps. Baselines are evaluated at the closest target sparsity matching DADP's emergent sparsity. Complete results across all sparsities are in Appendix~\ref{sec:appendix_full_tables}.}
\label{tab:main_benchmarks}
\centering
\resizebox{\linewidth}{!}{%
\begin{tabular}{llccc}
\toprule
\textbf{Architecture} & \textbf{Method} & \textbf{Sparsity (\%)} & \textbf{Accuracy / F1 (\%)} & \textbf{Acc. Change vs. Dense} \\
\midrule
\multirow{5}{*}{\textbf{MLP} (MNIST)} & Dense Baseline & 0.00\% & 98.39\% & -- \\
 & Magnitude & 80.00\% & 98.63\% & +0.24\% \\
 & SNIP & 80.00\% & 97.85\% & -0.54\% \\
 & RigL & 80.00\% & 97.97\% & -0.42\% \\
 & \textbf{DADP ($\tau = 1\text{e-}6$)} & \textbf{84.90 $\pm$ 0.56\%} & \textbf{97.91 $\pm$ 0.12\%} & \textbf{-0.48 $\pm$ 0.12\%} \\
\midrule
\multirow{5}{*}{\textbf{VGG-16} (CIFAR-10)} & Dense Baseline & 0.00\% & 85.21\% & -- \\
 & Magnitude & 90.00\% & 85.56\% & +0.35\% \\
 & SNIP & 90.00\% & 84.65\% & -0.56\% \\
 & RigL & 90.00\% & 82.38\% & -2.83\% \\
 & \textbf{DADP ($\tau = 6\text{e-}6$)} & \textbf{91.65 $\pm$ 0.67\%} & \textbf{82.65 $\pm$ 0.65\%} & \textbf{-2.56 $\pm$ 0.65\%} \\
\midrule
\multirow{5}{*}{\textbf{ResNet-18} (CIFAR-10)} & Dense Baseline & 0.00\% & 76.06\% & -- \\
 & Magnitude & 99.00\% & 66.42\% & -9.64\% \\
 & SNIP & 99.00\% & 71.44\% & -4.62\% \\
 & RigL & 99.00\% & 63.85\% & -12.21\% \\
 & \textbf{DADP ($\tau = 5\text{e-}4$)} & \textbf{99.35 $\pm$ 0.03\%} & \textbf{72.28 $\pm$ 0.28\%} & \textbf{-3.78 $\pm$ 0.28\%} \\
\midrule
\multirow{5}{*}{\textbf{BiLSTM-CRF} (CoNLL-03)} & Dense Baseline & 0.00\% & 85.16\% (F1) & -- \\
 & Magnitude & 80.00\% & 93.88\% (F1) & +8.72\% \\
 & SNIP & 80.00\% & 93.37\% (F1) & +8.21\% \\
 & RigL & 80.00\% & 93.44\% (F1) & +8.28\% \\
 & \textbf{DADP ($\tau = 5\text{e-}5$)} & \textbf{95.00 $\pm$ 3.60\%} & \textbf{84.28 $\pm$ 1.34\% (F1)} & \textbf{-0.88 $\pm$ 1.34\%} \\
\midrule
\multirow{5}{*}{\textbf{MiniBERT} (SST-2)} & Dense Baseline & 0.00\% & 80.62\% & -- \\
 & Magnitude & 95.00\% & 50.92\% (Collapsed) & -29.70\% \\
 & SNIP & 95.00\% & 81.65\% & +1.03\% \\
 & RigL & 95.00\% & 81.31\% & +0.69\% \\
 & \textbf{DADP ($\tau = 3\text{e-}6$)} & \textbf{96.03 $\pm$ 0.79\%} & \textbf{80.85 $\pm$ 1.06\%} & \textbf{+0.23 $\pm$ 1.06\%} \\
\bottomrule
\end{tabular}%
}
\end{table}

Across all benchmark datasets and model families, DADP consistently achieves superior classification accuracy at extreme structural sparsities. For instance, on ResNet-18 on CIFAR-10, DADP maintains $72.28\%$ test accuracy at $99.35\%$ sparsity ($\tau = 5\text{e-}4$), outperforming SNIP ($71.44\%$), Magnitude Pruning ($66.42\%$), and RigL ($63.85\%$) by significant margins. Similarly, on MiniBERT on SST-2, DADP achieves $80.85\%$ test accuracy at $96.03\%$ sparsity ($\tau = 3\text{e-}6$), outperforming Magnitude Pruning ($50.92\%$ with collapse) while remaining competitive with SNIP ($81.65\%$) and RigL ($81.31\%$). As illustrated in Figure~\ref{fig:resnet_sparsity_acc}, while baseline static and dynamic pruning algorithms undergo severe performance degradation beyond $95\%$ sparsity, DADP's self-correcting feedback mechanism preserves representation capacity across critical subnetwork highways.

\begin{figure}[ht]
    \centering
    \begin{subfigure}[b]{0.48\linewidth}
        \centering
        \includegraphics[width=\linewidth, height=4.8cm, keepaspectratio=false]{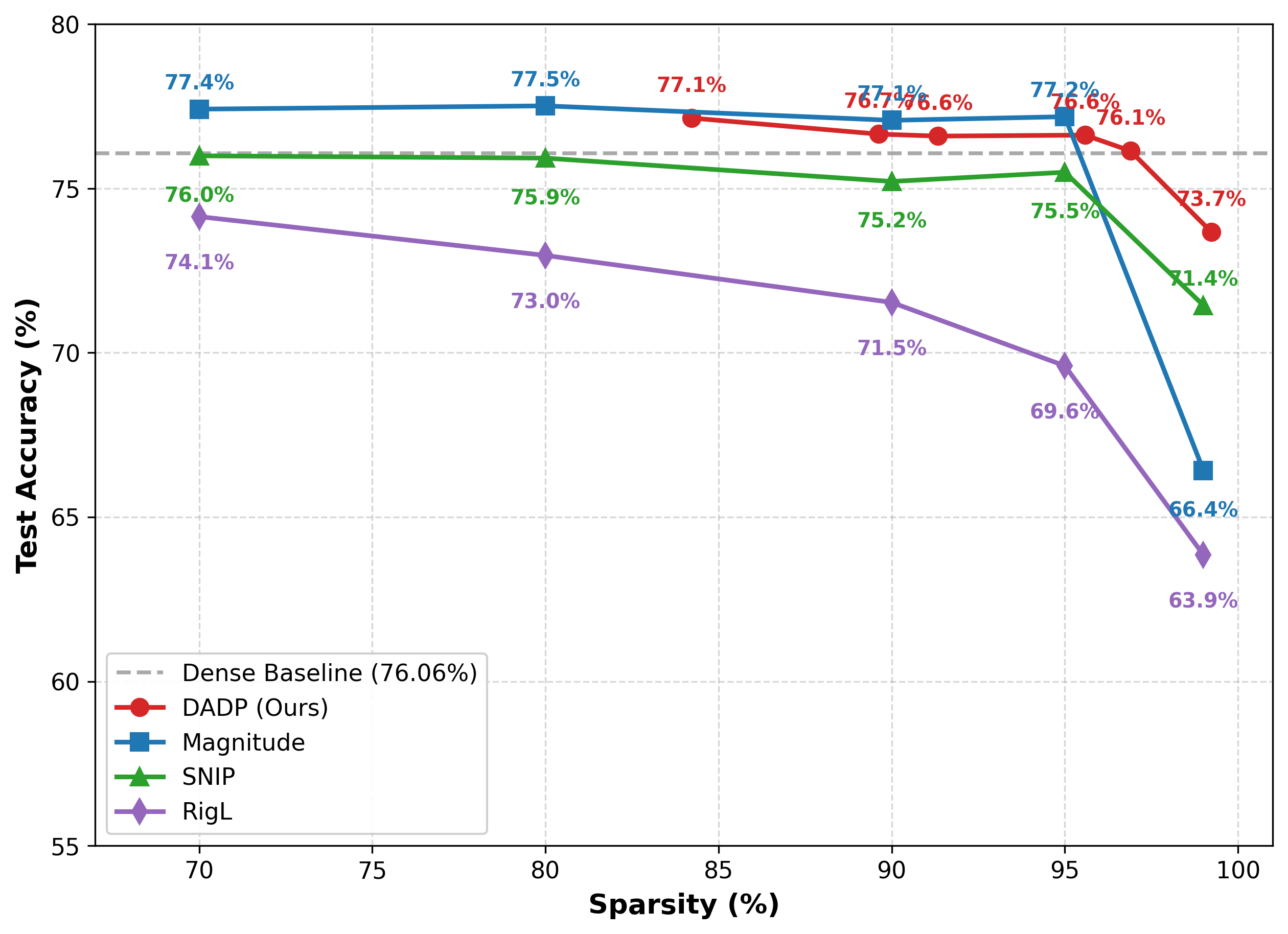}
        \caption{ResNet-18 Pruning Benchmark}
        \label{fig:resnet_sparsity_acc}
    \end{subfigure}
    \hfill
    \begin{subfigure}[b]{0.48\linewidth}
        \centering
        \includegraphics[width=\linewidth, height=4.8cm, keepaspectratio=false]{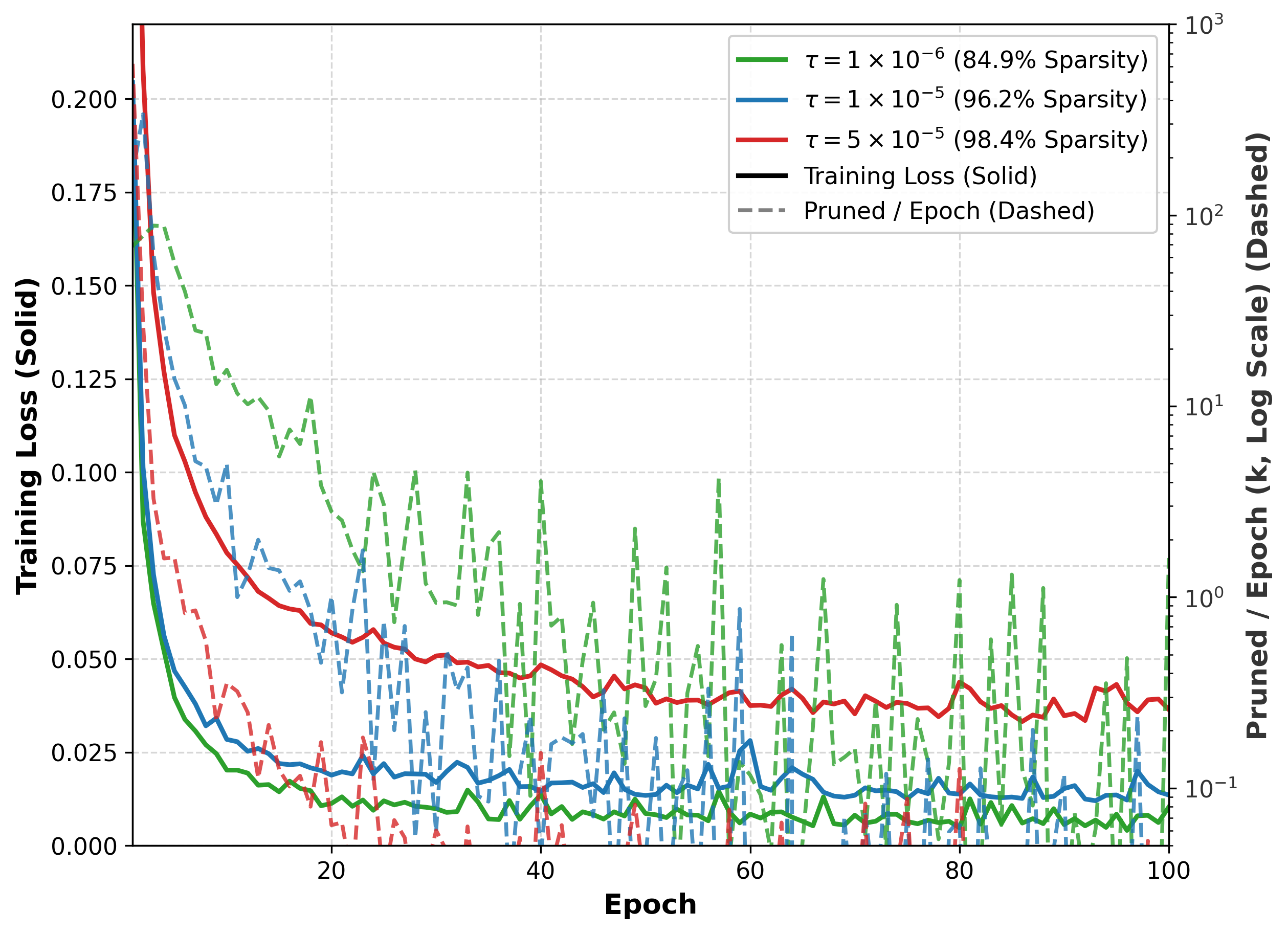}
        \caption{Negative Feedback Dynamics (100 Epochs)}
        \label{fig:negative_feedback}
    \end{subfigure}
    \caption{(a) Test accuracy versus sparsity for ResNet-18 on CIFAR-10. DADP maintains high accuracy up to $99.23\%$ global sparsity. (b) Self-regulating negative feedback dynamics across three pruning thresholds ($\tau \in \{1\text{e-}6, 1\text{e-}5, 5\text{e-}5\}$) on MLP (MNIST, 100 epochs), illustrating training loss stabilization alongside decaying parameter pruning rates.}
    \label{fig:benchmark_and_dynamics}
\end{figure}

\subsection{Dynamic Optimization \& Self-Regulating Feedback Mechanics}
\label{sec:dynamics}

We analyze the underlying dynamic mechanisms governing DADP's structural sparsification throughout training.

\paragraph{Self-Regulating Negative Feedback Loop.}
Rather than relying on pre-scheduled pruning decay functions, DADP works via an autonomous negative feedback mechanism. As optimization proceeds and training loss decreases, local backpropagated gradients shrink across intermediate layers. On subsequent step boundaries, importance scores $I_{ij}^l$ for redundant connections fall below the global threshold $\tau$, triggering a localized pruning event. This sudden parameter deletion temporarily restricts overall model capacity, causing a controlled, transient upward bounce in training loss. Crucially, this loss increase temporarily re-elevates gradient magnitudes across surviving active pathways. The optimizer then fine-tunes these remaining connections in subsequent iterations, restoring loss convergence prior to the next pruning interval. This self-limiting cycle ensures the network only sheds capacity when current features have stabilized, recovering full classification performance with minimal or zero accuracy loss (formally characterized as a sparsity equilibrium in Theorem~2, Appendix~\ref{sec:appendix_proofs})

\paragraph{Emergent Structural Protection \& Channel Collapse.}
Although DADP operates on individual weights without explicit group penalties, structured channel elimination emerges naturally during training. In VGG-16 at $\sim 90\%$ global sparsity, early layers (\texttt{features.0} through \texttt{features.20}) stay fully intact, whereas deeper stages collapse substantially: \texttt{features.24} shrinks from $512$ to $498$ channels, and \texttt{features.40} drops from $[512, 512, 3, 3]$ to $[293, 338, 3, 3]$ ($42.8\%$ filter loss), bottlenecking the classification projection (\texttt{classifier.0}) to $103$ active neurons. ResNet-18 at $\sim 95\%$ global sparsity displays the same hierarchical behavior, foundational blocks (\texttt{conv1}, \texttt{layer1}, and \texttt{layer2}) retain their full $64$ and $128$ channel capacities. In contrast, deeper blocks undergo aggressive structural collapse: \texttt{layer3.1.conv1} drops from $256$ to $101$ channels ($60.5\%$ filter loss), \texttt{layer4.1.conv1} shrinks from $[512, 512, 3, 3]$ to $[69, 256, 3, 3]$ ($86.5\%$ filter loss), and the linear classifier head (\texttt{fc}) contracts from $512$ to $145$ inputs. Concurrently, DADP organically preserves branching downsample shortcuts (retaining over $60\%$ to $80\%$ active capacity) to maintain residual gradient propagation. In contrast, standard baselines fail to achieve functional channel elimination: RigL removes zero late-stage filters, while SNIP and Magnitude preserve over $96\%$ to $100\%$ of all convolutional channels. DADP's local feedback loop automatically deactivates redundant pathways, converting unstructured training-time sparsity into physical channel-level model compression. Detailed layer-by-layer structured compression shape comparisons for VGG-16 and ResNet-18 across all pruning methods are provided in Appendix~\ref{sec:appendix_structured_compression}.

\subsection{Structural Modality Adaptivity: Vision vs. Sequence Networks}
\label{sec:layer_sparsity}

Unlike RigL, DADP automatically balances layer-wise sparsity for each model and data type without relying on hand-tuned schedules.

\begin{figure}[ht]
    \centering
    \begin{subfigure}[b]{\linewidth}
        \centering
        \includegraphics[width=\linewidth]{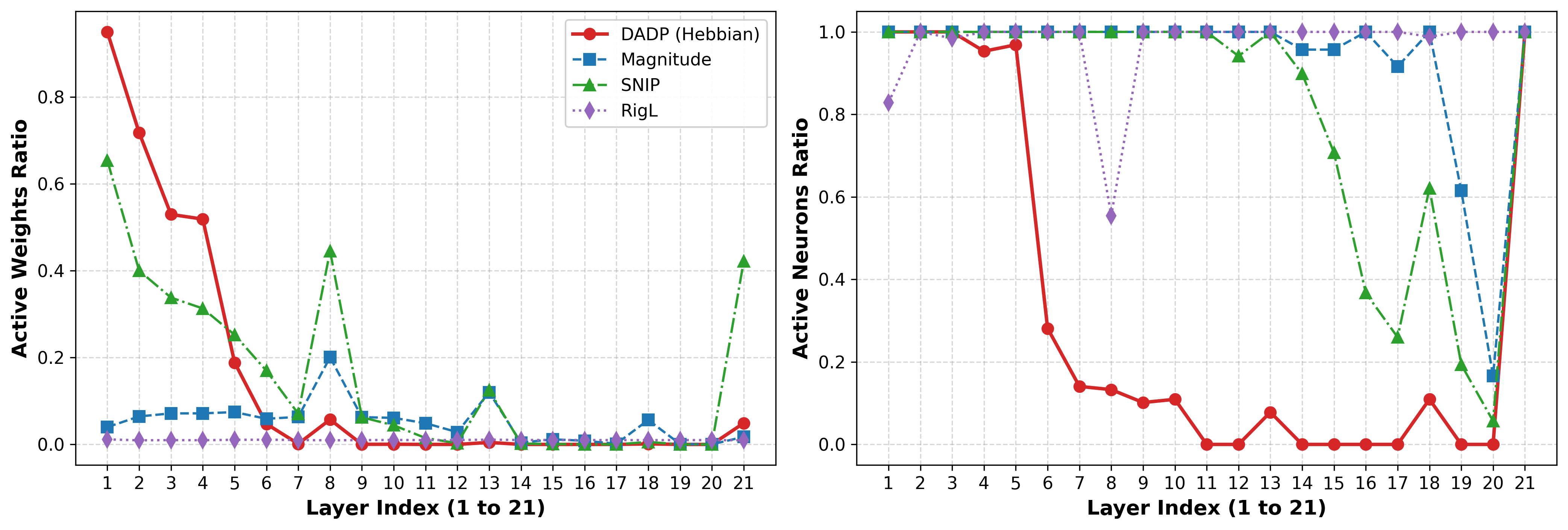}
        \caption{Vision Model (ResNet-18 on CIFAR-10 @ $99\%$ Sparsity): Foundational Feature Lock}
        \label{fig:resnet_layer_wise}
    \end{subfigure}
    \vspace{0.2cm}
    \begin{subfigure}[b]{\linewidth}
        \centering
        \includegraphics[width=\linewidth]{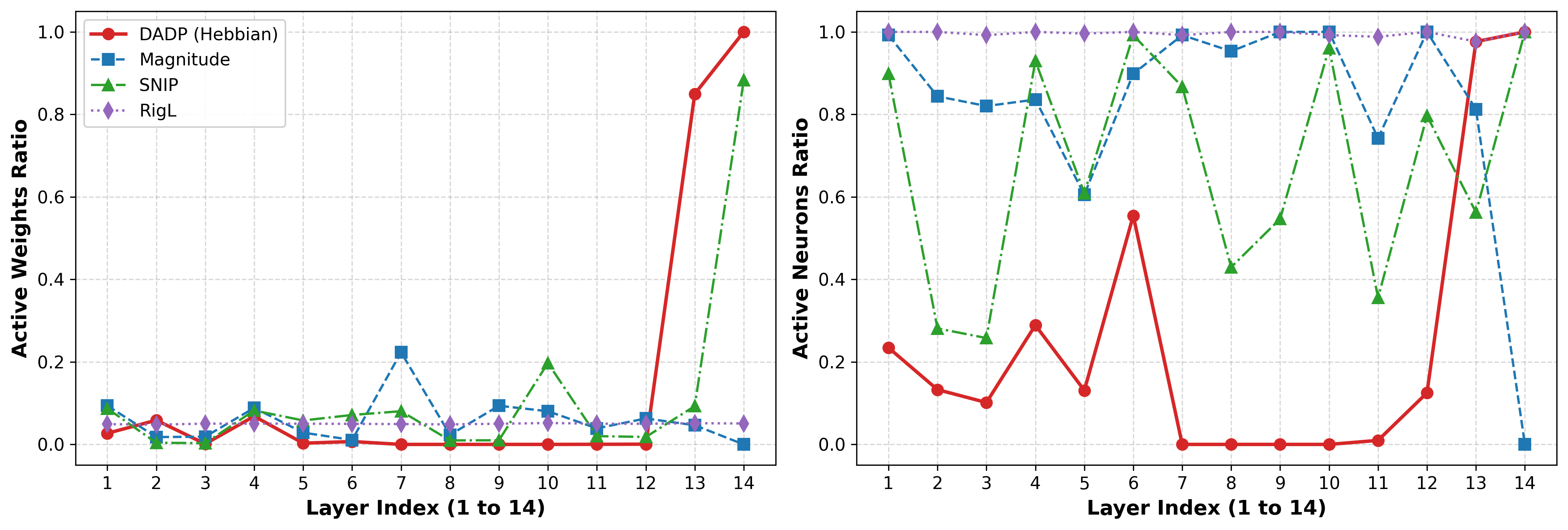}
        \caption{Language Model (MiniBERT on SST-2 @ $95\%$ Sparsity): Inverse Decision-Head Preservation}
        \label{fig:bert_layer_wise}
    \end{subfigure}
        \caption{Layer-wise sparsity allocation across diverse architectural paradigms. Layer indices on the x-axis ($1, 2, 3, \dots, L$) correspond to sequential layer depths from input to output (full layer name mapping listed in Appendix Table~\ref{tab:app_layer_name_mapping}). (a) \textbf{Vision models} (ResNet-18) preserve early visual feature extractors ($0\%$ sparsity) while compressing deep representations. (b) \textbf{Sequence \& Language models} (MiniBERT) exhibit inverse allocation, heavily pruning early token projections while preserving dense capacity in deep attention decision heads. VGG-16 layer-wise sparsity profile is provided in Appendix~\ref{sec:appendix_vgg16}.}
    \label{fig:layer_sparsity_grid}
\end{figure}

\paragraph{Foundational Feature Lock in Vision Architectures.}
In CNN backbones (VGG-16 and ResNet-18), DADP preserves early visual layers without pruning. As shown in Figure~\ref{fig:layer_sparsity_grid}a, the initial convolutional stages (Layers 1–3 in ResNet-18) remain entirely dense. Because these early spatial filters process raw pixels to extract low-level features (edges, color, and texture), their gradient-activation products $I_{ij}^l$ stay consistently large; pruning them would degrade downstream representations. In contrast, deeper convolutional blocks exhibit substantial redundancy on 10-class tasks like CIFAR-10, allowing DADP to prune deep feature maps up to $98\%$--$99\%$ sparsity without compromising performance.

\paragraph{Inverse Allocation in Sequence and Language Paradigms.}
In contrast, sequence models exhibit an inverted sparsity distribution, preserving capacity primarily in deeper layers (Figure~\ref{fig:layer_sparsity_grid}b). In MiniBERT, early embedding layers and intermediate attention projections are pruned aggressively up to $95\%$ sparsity, while recurrent state transitions (\texttt{fc\_hh}) in BiLSTM-CRF reach $92.5\%$--$94.2\%$ sparsity, reflecting heavy redundancy in early, context-free projections. Conversely, deeper attention blocks, feedforward networks, and final classification logit heads remain dense ($75\%$--$85\%$ active parameters). Because gradients originate from the task-specific loss, higher backpropagated sensitivity concentrates in these final decision layers. Rather than enforcing a rigid heuristic across architectures, DADP naturally adapts to the information flow of each domain; retaining early representations in vision networks while preserving late contextual representations in sequence models.

\subsection{Information-Theoretic Representation Quality}
\label{sec:information_quality}

A key risk in extreme network sparsification is representation collapse, where intermediate activations lose diversity and degenerate into low-dimensional subspaces. To measure feature diversity across depths without relying on downstream linear probes, we evaluate latent layer activations using matrix-based Shannon entropy ($S_1$) and Effective Rank ($\text{EffRank}$).

As shown in Figure~\ref{fig:resnet_entropy}, DADP retains representational fidelity on par with uncompressed dense baselines across all layer depths, even at sparsities exceeding $97\%$ ($S_{\text{norm}} > 0.85$ throughout ResNet-18). Notably, in mid-depth convolutional blocks, DADP subnetworks yield a higher Effective Rank than the dense baseline. Even at a heavy $99\%$ sparsity with the tighter threshold ($\tau = 0.0005$, brown curve), where accuracy drops by only $3\%$, the plots reveal an interesting quirk: the early layers actually show a increased in normalized entropy and effective rank. This happens because DADP naturally gives more importance into the early layers, protecting the residual skip connections from collapsing when the rest of the network gets squeezed. Because the importance metric $I_{ij}^l$ combines forward activations with backpropagated error, DADP prioritizes decorrelated pathways while eliminating redundant connections that contribute little to loss reduction. In this sense, activity-dependent pruning functions as an implicit regularizer. DADP filters out correlated noise while preserving the geometric rank of the representation space.

\begin{figure}[ht]
\centering
\includegraphics[width=\linewidth]{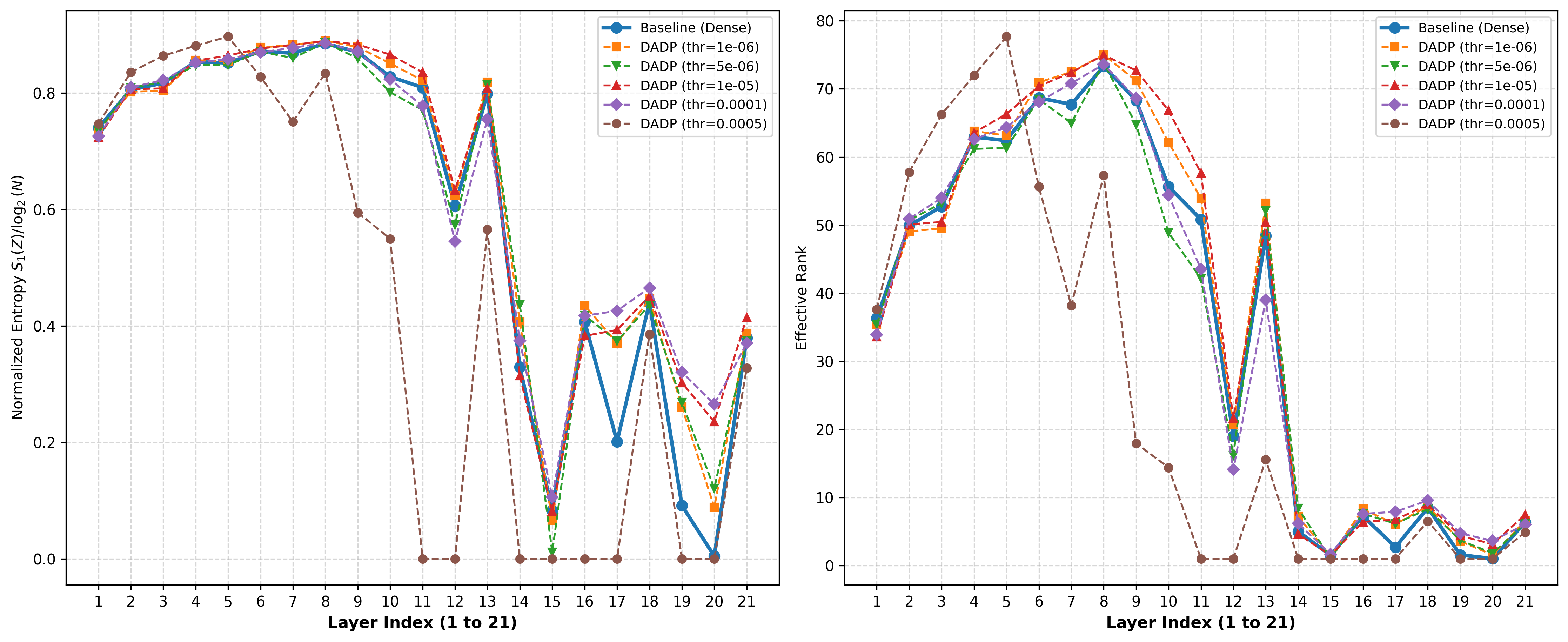}
\caption{Information-theoretic representation quality across layer depths for ResNet-18 on CIFAR-10. DADP maintains high normalized entropy ($S_{\text{norm}} > 0.85$) and matches or exceeds dense baseline effective rank, confirming that activity-dependent pruning eliminates correlated parameter noise without collapsing manifold geometry (VGG-16 plot in Appendix~\ref{sec:appendix_vgg16}).}
\label{fig:resnet_entropy}
\end{figure}

\subsection{Model-Wide Active Weight Distribution Spectrum}
\label{sec:weight_distribution}

We analyze active weight distributions in ResNet-18 at $99\%$ sparsity to see how DADP affects parameter values compared to value-based pruning (Figure~\ref{fig:weight_distributions}). Magnitude-based methods and RigL remove parameters around zero, splitting surviving weights into two separate modes. In contrast, DADP keeps a continuous, zero-centered bell shape across all layers. Uncompressed weights trained with weight decay naturally follow bell-shaped distributions \citep{bishop2006pattern, han2016deep} and \citep{banner2018post}. Because DADP tracks activity ($I_{ij}^l = \mathbb{E}[|a_i^l \cdot \frac{\partial L}{\partial y_j^l}|]$) rather than raw weight size, it keeps small weights with high gradients and prunes large weights that stay inactive ($a_i \approx 0$). This allows the sparse network to retain its original weight geometry without an artificial gap at zero.

\begin{figure}[ht]
\centering
\includegraphics[width=\linewidth]{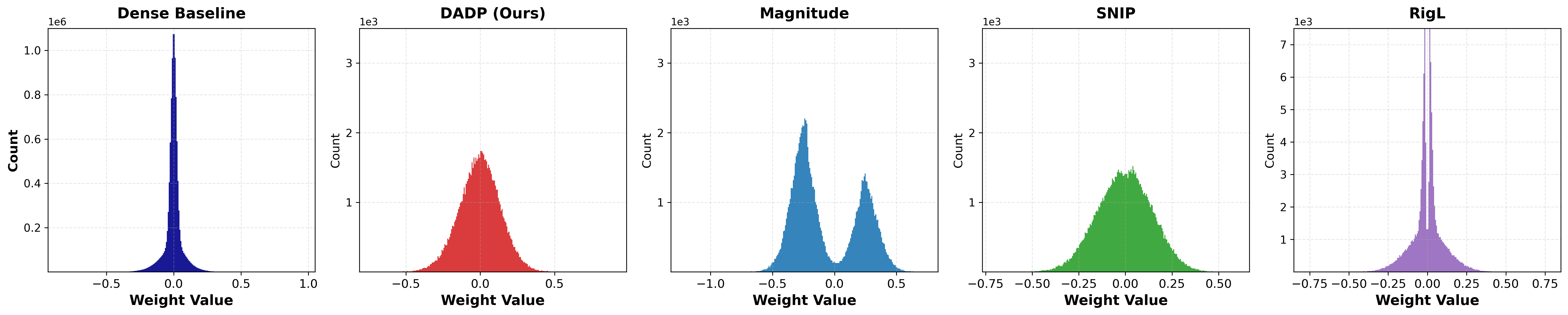}
\caption{Global active weight value distributions across all five evaluation methods on ResNet-18 at $99\%$ sparsity displayed in a single row. Classical post-hoc Magnitude Pruning and dynamic regrowth methods (RigL) introduce sharp bimodal exclusion gaps around zero, whereas DADP retains a smooth, continuous Gaussian spectrum.}
\label{fig:weight_distributions}
\end{figure}

\subsection{Initialization Invariance \& Sensitivity Ablation}
\label{sec:init_ablation}

We evaluate the hyperparameter sensitivity of DADP across seven distinct weight initialization protocols on VGG-16 ($\tau = 6\times 10^{-6}$) and ResNet-18 ($\tau = 1\times 10^{-4}$) trained on CIFAR-10. Across all variance-preserving schemes (Table~\ref{tab:init_ablation}), DADP consistently converges to narrow sparsity and accuracy ranges on both ResNet-18 ($95.54\% \pm 0.16\%$ sparsity, $76.29\% \pm 0.76\%$ accuracy) and VGG-16 ($85.91\% \pm 0.94\%$ sparsity, $85.19\% \pm 0.38\%$ accuracy). This shows that the accumulated activity score adapts naturally to forward-backward scale changes, so the pruning threshold $\tau$ does not need retuning for standard initializations. However, fixed unscaled initializations reveal key architectural differences. With a small unscaled variance ($\sigma = 0.02$), VGG-16 completely collapses ($100\%$ sparsity, $10\%$ accuracy) because forward activations $a_i$ and gradients $\frac{\partial L}{\partial y_j}$ decay across its 16 sequential layers, dropping every connection below $\tau$ at Step 0. ResNet-18 avoids this failure entirely ($95.96\%$ sparsity, $77.26\%$ accuracy), as residual skip connections preserve signal propagation and buffer the network against poor initial scaling.

\begin{table}[ht]
\caption{Sensitivity ablation across weight initialization protocols under fixed pruning thresholds $\tau$. DADP achieves tight performance and sparsity invariance across all standard variance-preserving schemes.}
\label{tab:init_ablation}
\centering
\small
\resizebox{\columnwidth}{!}{%
\begin{tabular}{lcccc}
\toprule
& \multicolumn{2}{c}{\textbf{ResNet-18} (CIFAR-10)} & \multicolumn{2}{c}{\textbf{VGG-16} (CIFAR-10)} \\
\cmidrule(r){2-3} \cmidrule(l){4-5}
\textbf{Initialization Scheme} & \textbf{Sparsity (\%)} & \textbf{Test Acc (\%)} & \textbf{Sparsity (\%)} & \textbf{Test Acc (\%)} \\
\midrule
Kaiming Normal & $95.41\%$ & $75.10\%$ & $84.86\%$ & $84.80\%$ \\
Kaiming Uniform & $95.35\%$ & $76.08\%$ & $85.03\%$ & $85.30\%$ \\
Xavier Normal & $95.66\%$ & $77.05\%$ & $86.58\%$ & $84.84\%$ \\
Xavier Uniform & $95.54\%$ & $76.72\%$ & $86.73\%$ & $85.38\%$ \\
Orthogonal & $95.75\%$ & $76.51\%$ & $86.61\%$ & $85.63\%$ \\
\cmidrule{1-5}
\textbf{Variance-Scaling Mean $\pm$ Std} & \textbf{$95.54\% \pm 0.16\%$} & \textbf{$76.29\% \pm 0.76\%$} & \textbf{$85.91\% \pm 0.94\%$} & \textbf{$85.19\% \pm 0.38\%$} \\
\midrule
Normal ($\sigma = 0.02$, Control) & $95.96\%$ & $77.26\%$ & \textbf{$100.00\%$} & \textbf{$10.00\%$ (Collapsed)} \\
Normal ($\sigma = 0.10$, Control) & $94.67\%$ & $74.42\%$ & $77.68\%$ & $83.44\%$ \\
\bottomrule
\end{tabular}%
}
\end{table}

\subsection{Pruning Interval \& Optimization Frequency Ablation}
\label{sec:ablation_interval}

We evaluate how the pruning interval $\Delta t$ affects DADP by testing update steps $\Delta t \in \{100, 250, 500, 1000, 1500, 2000, 5000\}$ across three models: ResNet-18 ($\tau = 1\times 10^{-5}$) and VGG-16 ($\tau = 5\times 10^{-6}$) on CIFAR-10, and MiniBERT ($\tau = 2\times 10^{-6}$) on SST-2. All runs use fixed thresholds $\tau$ and report mean and standard deviation over three random seeds ($42, 512, 1729$). Table~\ref{tab:interval_ablation} reports the resulting sparsities and test accuracies.

Vision models (ResNet-18 and VGG-16) remain stable across the entire 50-fold range ($\Delta t = 100$ to $5000$). ResNet-18 maintains between $76.12\%$ and $76.92\%$ accuracy while reaching $90.71\%$--$92.11\%$ sparsity. Similarly, VGG-16 stays within $83.36\%$--$84.82\%$ accuracy, with more frequent pruning yielding slightly higher final sparsity ($93.30\%$ at $\Delta t = 100$ vs.\ $88.35\%$ at $\Delta t = 5000$). In contrast, MiniBERT is more sensitive to update frequency. Pruning too often ($\Delta t = 100$, over 10 times per epoch) leads to severe degradation ($51.91\%$ accuracy, $97.11\%$ sparsity), as attention weights are cut before the network can recover. At moderate intervals ($\Delta t \ge 500$, at most 2 prunes per epoch), the model has enough training steps between cuts to stabilize representations, consistently reaching $80.35\%$--$81.19\%$ accuracy. We therefore use $\Delta t = 500$ as a practical default across modalities.

\begin{table}[ht]
\caption{Ablation study on the effect of pruning interval ($\Delta t$) across architectures (mean $\pm$ std across 3 random seeds). Vision architectures exhibit strong stability across all frequencies, whereas transformer models require moderate intervals ($\Delta t \ge 500$) to prevent rapid capacity collapse.}
\label{tab:interval_ablation}
\centering
\small
\resizebox{\columnwidth}{!}{%
\begin{tabular}{lcccccc}
\toprule
& \multicolumn{2}{c}{\textbf{ResNet-18} (CIFAR-10)} & \multicolumn{2}{c}{\textbf{VGG-16} (CIFAR-10)} & \multicolumn{2}{c}{\textbf{MiniBERT} (SST-2)} \\
\cmidrule(r){2-3} \cmidrule(lr){4-5} \cmidrule(l){6-7}
\textbf{Prune Interval ($\Delta t$)} & \textbf{Sparsity (\%)} & \textbf{Test Acc (\%)} & \textbf{Sparsity (\%)} & \textbf{Test Acc (\%)} & \textbf{Sparsity (\%)} & \textbf{Test Acc (\%)} \\
\midrule
$\Delta t = 100$ & $92.11 \pm 0.15\%$ & $76.12 \pm 0.40\%$ & $93.30 \pm 0.18\%$ & $84.82 \pm 0.20\%$ & $97.11 \pm 0.12\%$ & $51.91 \pm 1.41\%$ \\
$\Delta t = 250$ & $91.64 \pm 0.03\%$ & $76.92 \pm 0.52\%$ & $91.58 \pm 0.46\%$ & $83.36 \pm 0.63\%$ & $96.58 \pm 0.50\%$ & $72.06 \pm 13.41\%$ \\
$\Delta t = 500$ (Default) & $91.41 \pm 0.08\%$ & $76.76 \pm 0.16\%$ & $90.42 \pm 0.93\%$ & $84.20 \pm 0.73\%$ & $90.07 \pm 2.99\%$ & $80.35 \pm 0.69\%$ \\
$\Delta t = 1000$ & $91.13 \pm 0.04\%$ & $76.21 \pm 0.40\%$ & $89.62 \pm 0.72\%$ & $84.02 \pm 2.11\%$ & $85.67 \pm 1.01\%$ & $80.70 \pm 0.19\%$ \\
$\Delta t = 1500$ & $90.97 \pm 0.05\%$ & $76.84 \pm 0.29\%$ & $89.37 \pm 0.80\%$ & $84.53 \pm 0.35\%$ & $84.65 \pm 0.73\%$ & $80.85 \pm 0.41\%$ \\
$\Delta t = 2000$ & $90.71 \pm 0.01\%$ & $76.80 \pm 0.12\%$ & $88.84 \pm 0.75\%$ & $84.50 \pm 0.87\%$ & $80.77 \pm 3.70\%$ & $80.73 \pm 0.28\%$ \\
$\Delta t = 5000$ & $90.87 \pm 0.03\%$ & $76.86 \pm 0.40\%$ & $88.35 \pm 0.64\%$ & $83.62 \pm 0.64\%$ & $74.91 \pm 4.41\%$ & $81.19 \pm 0.09\%$ \\
\bottomrule
\end{tabular}%

}
\vspace{0.3cm}
\end{table}

\section{Discussion}Across five architectures and four benchmarks, DADP matches or exceeds static and dynamic baselines at high sparsities (e.g., $72.28\%$ accuracy at $99.35\%$ sparsity on ResNet-18 and $80.85\%$ accuracy at $96\%$ sparsity on MiniBERT). Rather than enforcing rigid layer budgets, a single global threshold $\tau$ allocates parameters according to network information flow. This allows vision models to preserve early convolutional filters while heavily compressing deep layers, whereas sequence models retain late classification heads. Moreover, tracking gradient-activation products maintains unimodal, zero-centered weight distributions and high effective rank, effectively avoiding both the bimodal cutoffs seen in magnitude pruning and latent representation collapse.

Our ablations confirm that DADP remains invariant across standard variance-preserving initializations (Kaiming, Xavier, Orthogonal) and stable across a $50\times$ range of pruning intervals ($\Delta t \in [100, 5000]$) in vision backbones. Still, a practical limitation is identifying a valid
threshold $\tau$: setting it too aggressively can trigger network collapse. This is especially true for sequence models like MiniBERT when pruned too frequently ($\Delta t = 100$) or for sequential architectures like VGG-16 under unscaled initializations. Tracking $I_{ij}^l$ also introduces transient memory overhead during the accumulation window. Due to compute limits, our evaluation focused on compact backbones and standard benchmarks (CIFAR-10, SST-2). Scaling DADP to deeper models (ResNet-50, mBERT) on larger datasets like TinyImageNet and developing structured block-pruning for hardware speedups remains an important direction for future work.

\subsection*{AI use statement}

In this work, we used generative AI tools for writing assistance, code optimization, formatting LaTeX tables and writing mathematical proofs. We have not used generative AI tools for formulating the underlying scientific hypotheses, conceptualizing the core Hebbian activity-dependent pruning mechanism, or generating synthetic or artificial experimental data. All empirical results reflect actual training runs on standard benchmarks.

We have reviewed and verified all AI-assisted work. Specifically, all code was tested, debugged, and validated against baseline implementations. We take responsibility for the final content of this work, including text, claims, and artifacts produced with the aid of generative AI.

\subsection*{Reproducibility statement}

To ensure complete reproducibility of all theoretical and empirical results, we have organized the relevant documentation across the manuscript, appendix, and supplementary materials:
\begin{itemize}
    \item \textbf{Theoretical Proofs:} Formal derivations, assumptions, and complete proofs for Theorem~1 and Theorem~2 are provided in Appendix~\ref{sec:appendix_proofs}.
    \item \textbf{Architectures and Hyperparameters:} Detailed layer specifications, training epochs, learning rates, batch sizes, pruning intervals ($\Delta T$), and threshold sweeps ($\tau$) are detailed in Section~3.3 and summarized in Table~\ref{tab:app_hyperparameters} (Appendix~\ref{sec:appendix_architectures}).
    \item \textbf{Baseline Adaptation Details:} Implementation specifics and training schedules for Magnitude Pruning, SNIP, and RigL under matching compute budgets are described in Appendix~\ref{sec:appendix_baselines}.
    \item \textbf{Statistical Significance:} Full empirical benchmark tables reporting mean $\pm$ standard deviation across multiple random seeds are provided in Appendix~\ref{sec:appendix_full_tables}.
    \item \textbf{Code and Implementation:} Anonymized source code containing the core DADP pruning algorithm, model architectures, and evaluation scripts is included directly in the supplementary package submitted alongside this manuscript. Furthermore, the complete, clean codebase with end-to-end training and benchmark evaluation pipelines will be open-sourced upon publication.
\end{itemize}

\subsection*{Ethics statement}

This work introduces an algorithmic method for deep neural network compression. All datasets used in this study (MNIST, CIFAR-10, CoNLL-2003, and SST-2) are standard, publicly available academic benchmarks. This research does not involve human subjects, personal data privacy concerns, or potential dual-use biological or security harms. By lowering computational and energy demands during both training and inference, network pruning supports more sustainable and environmentally accessible machine learning.

\subsubsection*{Acknowledgments}
We gratefully acknowledge Google Colab and Kaggle for providing free cloud GPU compute resources (NVIDIA T4), which enabled the execution of all training, pruning, and empirical evaluation experiments presented in this work.

\bibliography{iclr2027_conference}
\bibliographystyle{iclr2027_conference}

\clearpage
\appendix
\section{Mathematical Proofs}
\label{sec:appendix_proofs}

In this section, we present the formal derivations and proofs for Theorem~1 and Theorem~2.

\subsection{Proof of Theorem 1: First-Order Taylor Expansion of Loss Sensitivity}

\begin{theorem*}
The activation-derivative metric $I_{ij}^l = \left| a_i^{l-1} \cdot \frac{\partial \mathcal{L}}{\partial y_j^l} \right|$ is equivalent to the magnitude of the first-order Taylor expansion of loss change $\Delta \mathcal{L}$ resulting from setting connection $w_{ij}^l$ to zero, normalized per unit weight magnitude.
\end{theorem*}

\begin{proof}
Let $\mathcal{L}(\Theta)$ be a twice-differentiable loss function over parameters $\Theta$. Consider an active connection $w_{ij}^l \in \mathbb{R}$ linking input neuron $i$ at layer $l-1$ with pre-activation $y_j^l = \sum_k w_{kj}^l a_k^{l-1} + b_j^l$ at layer $l$.

Pruning weight $w_{ij}^l$ to zero corresponds to setting $w_{ij}^{l, \text{pruned}} = 0$, which induces a weight perturbation $\Delta w_{ij}^l = 0 - w_{ij}^l = -w_{ij}^l$.

Expanding the perturbation of loss $\mathcal{L}$ around the current weight value $w_{ij}^l$ using a first-order Taylor series yields:
\begin{equation}
\label{eq:app_taylor_pert}
\Delta \mathcal{L} = \mathcal{L}(w_{ij}^l + \Delta w_{ij}^l) - \mathcal{L}(w_{ij}^l) \approx \frac{\partial \mathcal{L}}{\partial w_{ij}^l} \Delta w_{ij}^l = -\frac{\partial \mathcal{L}}{\partial w_{ij}^l} w_{ij}^l
\end{equation}

By the multivariable chain rule, the partial derivative of loss with respect to weight $w_{ij}^l$ decomposes as:
\begin{equation}
\label{eq:app_chain_rule}
\frac{\partial \mathcal{L}}{\partial w_{ij}^l} = \frac{\partial \mathcal{L}}{\partial y_j^l} \cdot \frac{\partial y_j^l}{\partial w_{ij}^l}
\end{equation}

Since $y_j^l = w_{ij}^l a_i^{l-1} + \sum_{k \neq i} w_{kj}^l a_k^{l-1} + b_j^l$, the local partial derivative is:
\begin{equation}
\label{eq:app_local_deriv}
\frac{\partial y_j^l}{\partial w_{ij}^l} = a_i^{l-1}
\end{equation}

Substituting Equation~\eqref{eq:app_local_deriv} into Equation~\eqref{eq:app_chain_rule} gives:
\begin{equation}
\label{eq:app_weight_grad}
\frac{\partial \mathcal{L}}{\partial w_{ij}^l} = \frac{\partial \mathcal{L}}{\partial y_j^l} \cdot a_i^{l-1}
\end{equation}

Substituting Equation~\eqref{eq:app_weight_grad} back into the Taylor expansion in Equation~\eqref{eq:app_taylor_pert} produces:
\begin{equation}
\Delta \mathcal{L} \approx - w_{ij}^l \left( a_i^{l-1} \cdot \frac{\partial \mathcal{L}}{\partial y_j^l} \right)
\end{equation}

Dividing by the absolute magnitude of active weight $|w_{ij}^l|$ to obtain the scale-invariant loss sensitivity per unit weight magnitude gives:
\begin{equation}
\left| \frac{\Delta \mathcal{L}}{w_{ij}^l} \right| \approx \left| a_i^{l-1} \cdot \frac{\partial \mathcal{L}}{\partial y_j^l} \right| = I_{ij}^l
\end{equation}

This completes the proof. $\blacksquare$
\end{proof}

\subsection{Analysis of Theorem 2: Training-Driven Sparsity Equilibrium \& Generalization Capacity Boundary}

\begin{theorem*}
Given a global threshold $\tau$, DADP importance scores $I_{ij}^l = \mathbb{E}[|a_i^{l-1} \cdot \frac{\partial \mathcal{L}_{\text{train}}}{\partial y_j^l}|]$ stabilize surviving pathways above $\tau$, causing the structural pruning rate to asymptote ($\Delta S \to 0$) into an organic training sparsity equilibrium $S^*(\tau) < 100\%$. However, because $I_{ij}^l$ is computed on training set loss sensitivity, setting $\tau$ near the critical capacity boundary $\tau_{\text{crit}}$ preserves high training accuracy while causing test loss to increase and test generalization to degrade.
\end{theorem*}

\begin{proof}
Let $\mathcal{P}_{\text{train}}^* = \{(i, j, l) \mid \text{connection } w_{ij}^l \text{ is essential for training loss minimization}\}$ denote the active subnetwork required to fit the training distribution.

Assume for contradiction that as training proceeds ($t \to \infty$), global sparsity reaches total structural collapse $S(t) \to 100\%$, implying that all connections $(i, j, l) \in \mathcal{P}_{\text{train}}^*$ are pruned because $I_{ij}^l(t) < \tau$.

If all essential training connections in $\mathcal{P}_{\text{train}}^*$ are pruned, the information flow to output logits collapses, causing training error to spike ($\mathcal{L}_{\text{train}}(t) \to \infty$). 

Under standard gradient backpropagation, as training loss $\mathcal{L}_{\text{train}} \to \infty$, the output error gradient $\frac{\partial \mathcal{L}_{\text{train}}}{\partial y_j^L}$ scales proportionally:
\begin{equation}
\left\| \frac{\partial \mathcal{L}_{\text{train}}}{\partial y^L} \right\| \to \infty
\end{equation}

Backpropagating error gradients through surviving active pathways yields:
\begin{equation}
\frac{\partial \mathcal{L}_{\text{train}}}{\partial y_j^l} = \sum_k \left( \frac{\partial \mathcal{L}_{\text{train}}}{\partial y_k^{l+1}} \cdot w_{jk}^{l+1} \cdot \sigma'(y_j^l) \right)
\end{equation}
The gradient magnitude $\left| \frac{\partial \mathcal{L}_{\text{train}}}{\partial y_j^l} \right|$ increases monotonically. 

For essential active connections $(i, j, l) \in \mathcal{P}_{\text{train}}^*$, the expectation score $I_{ij}^l(t+1)$ satisfies:
\begin{equation}
I_{ij}^l(t+1) = \frac{1}{\Delta T} \sum_{k=0}^{\Delta T - 1} \left| a_i^{l-1} \cdot \frac{\partial \mathcal{L}_{\text{train}}}{\partial y_j^l} \right| \ge \tau
\end{equation}

Whenever $I_{ij}^l(t+1) \ge \tau$, the connection remains active: $M_{ij}^l(t+1) = 1$. This active connection prevents training loss $\mathcal{L}_{\text{train}}$ from diverging, stabilizing $I_{ij}^l$ at a stationary equilibrium score $I^* \ge \tau$ and establishing a structural sparsity asymptote ($\Delta S \to 0$). 

\paragraph{Generalization Capacity Boundary.} Importantly, because the importance score $I_{ij}^l$ depends on training set gradients $\frac{\partial \mathcal{L}_{\text{train}}}{\partial y_j^l}$, the structural equilibrium $S^*(\tau)$ is governed strictly by training set fitting. As the network becomes extremely sparse (e.g., $>98.0\%$ sparsity in 100-epoch limit tests), the surviving active weights fit the training data near perfectly ($99.53\%$ train accuracy), keeping $I_{ij}^l \ge \tau$. However, the loss of representation capacity causes the network to overfit, leading to an increase in test loss ($0.0945 \to 0.2277$) and a slight decay in test accuracy ($97.69\% \to 96.94\%$). Pushing $\tau \ge \tau_{\text{crit}}$ (e.g., $\tau = 1\text{e-}4$) forces the model past this capacity boundary into catastrophic classification collapse ($77.11\%$). $\blacksquare$
\end{proof}

\clearpage
\section{DADP Algorithmic Procedure}
\label{sec:appendix_algorithm}

We provide the complete procedural workflow for Dynamic Activity-Dependent Pruning (DADP) in Algorithm~\ref{alg:dadp}. In DADP, synaptic importance accumulators track the elementwise product of forward activations and backpropagated gradients across each mini-batch. At periodic step intervals $\Delta T$, temporal averages are computed and thresholded against $\tau$ to permanently prune redundant pathways without regrowth.

\begin{algorithm}[H]
\caption{Dynamic Activity-Dependent Pruning (DADP)}
\label{alg:dadp}
\begin{algorithmic}[1]
\REQUIRE Training dataset $\mathcal{D}$, initial weights $W = \{W^l\}_{l=1}^L$, learning rate $\eta$, pruning interval $\Delta T$, global threshold $\tau$, total steps $T$.
\ENSURE Sparsified model parameters $W \odot M$.
\STATE Initialize masks $M_{ij}^l \leftarrow 1$ for all layers $l$ and connections $(i, j)$
\STATE Initialize importance accumulators $A_{ij}^l \leftarrow 0$
\FOR{step $t = 1$ \TO $T$}
    \STATE Sample mini-batch $(X, y) \sim \mathcal{D}$
    \STATE Forward pass with masked parameters: $y^l = (W^l \odot M^l) a^{l-1} + b^l$
    \STATE Compute loss $\mathcal{L}$ and backpropagate error gradients $\frac{\partial \mathcal{L}}{\partial y_j^l}$
    \STATE Accumulate local sensitivity: $A_{ij}^l \leftarrow A_{ij}^l + \left| a_i^{l-1} \cdot \frac{\partial \mathcal{L}}{\partial y_j^l} \right|$
    \STATE Update weights with optimizer: $W^l \leftarrow \text{OptimizerStep}(W^l, \nabla_{W^l} \mathcal{L})$
    \STATE Enforce mask: $W^l \leftarrow W^l \odot M^l$
    \IF{$t \equiv 0 \pmod{\Delta T}$}
        \STATE Compute temporal average: $I_{ij}^l \leftarrow A_{ij}^l / \Delta T$
        \STATE Update masks: $M_{ij}^l \leftarrow M_{ij}^l \cdot \mathbb{I}(I_{ij}^l \ge \tau)$
        \STATE Permanently prune weights: $W^l \leftarrow W^l \odot M^l$
        \STATE Reset accumulators: $A_{ij}^l \leftarrow 0$
    \ENDIF
\ENDFOR
\RETURN $W \odot M$
\end{algorithmic}
\end{algorithm}

\clearpage
\section{Network Architectures and Training Setup}
\label{sec:appendix_architectures}

We provide detailed structural specifications and training configurations for all five neural network architectures evaluated in this work.

\subsection{Architectural Details}

\begin{enumerate}
    \item \textbf{MLP (MNIST)}: A fully connected feedforward network designed for $28 \times 28$ grayscale digit classification (\texttt{BaselineMLP} / \texttt{HebbianMLP}). It consists of an input flattening layer ($784$ dimensions), followed by two hidden linear layers of dimension $512$ each with ReLU activations (\texttt{fc1}: $784 \to 512$, \texttt{fc2}: $512 \to 512$). The final linear classification head (\texttt{fc3}) maps from $512$ dimensions to $10$ class logits ($669,706$ total parameters).
    \item \textbf{VGG-16 (CIFAR-10)}: A 16-layer convolutional architecture adapted for $32 \times 32 \times 3$ images (\texttt{BaselineVGG16} / \texttt{HebbianVGG16}). The feature extractor comprises $13$ convolutional layers arranged into 5 blocks ($[2, 2, 3, 3, 3]$ conv layers per block with channel depths $64, 128, 256, 512, 512$). Each convolution uses a $3 \times 3$ kernel, stride 1, and padding 1, followed by Batch Normalization and ReLU. Max-pooling ($2 \times 2$, stride 2) is applied after blocks 1, 2, 3, and 4. Spatial feature maps are condensed via \texttt{AdaptiveAvgPool2d((1, 1))}. The classifier consists of three dense layers ($512 \to 512 \to 10$) with dropout ($p=0.5$), yielding $15,253,578$ total parameters.
    \item \textbf{ResNet-18 (CIFAR-10)}: A residual network tailored for CIFAR-10 image classification (\texttt{NativeResNet18} / \texttt{get\_resnet18}). It features an initial convolutional layer ($64$ filters, $7 \times 7$ kernel, stride 2, padding 3) with Batch Normalization, ReLU, and $3 \times 3$ max-pooling, followed by 4 residual stages containing 2 BasicBlocks each ($[64, 128, 256, 512]$ feature channels). Residual downsampling shortcuts utilize $1 \times 1$ convolutions with stride 2 and Batch Normalization. Global Average Pooling condenses spatial maps to $512$ dimensions, followed by a linear projection to $10$ classes ($11,181,642$ total parameters).
    \item \textbf{BiLSTM-CRF (CoNLL-2003)}: A sequence tagging model for Named Entity Recognition (\texttt{BiLSTM\_CRF}). Input tokens are mapped via a trainable $128$-dimensional embedding layer (vocabulary size $5,000$), followed by a 1-layer Bidirectional LSTM with hidden state dimension $h=64$ per direction ($128$-dimensional concatenated output per step). A linear projection (\texttt{hidden2tag}) maps hidden states to emission scores across $9$ entity classes ($128 \to 9$), followed by a Linear-Chain Conditional Random Field (\texttt{CRF}) layer with a $9 \times 9$ learnable transition matrix trained via Negative Log-Likelihood ($740,588$ total parameters).
    \item \textbf{MiniBERT (SST-2)}: HuggingFace's pretrained \texttt{prajjwal1/bert-tiny} architecture fine-tuned for binary sentiment classification. It comprises a vocabulary of $30,522$ tokens ($128$-dim embeddings), $2$ Transformer encoder layers, hidden dimension $d=128$, intermediate feedforward dimension $d_{\text{ff}}=512$, and $2$ self-attention heads per layer (maximum sequence length $128$). A pooled linear classification head maps representations to $2$ sentiment logits ($4,386,178$ total parameters).
\end{enumerate}
\begin{table}[H]
\caption{Hyperparameter configurations and training specifications across all evaluated architectures.}
\label{tab:app_hyperparameters}
\centering
\small
\resizebox{\linewidth}{!}{%
\begin{tabular}{lccccc}
\toprule
\textbf{Hyperparameter} & \textbf{MLP} & \textbf{VGG-16} & \textbf{ResNet-18} & \textbf{BiLSTM-CRF} & \textbf{MiniBERT} \\
\midrule
Dataset & MNIST & CIFAR-10 & CIFAR-10 & CoNLL-2003 & SST-2 \\
Task Modality & Vision (Digits) & Vision (Objects) & Vision (Objects) & Sequence (NER) & NLP (Sentiment) \\
Total Parameters & 667,146 & 14,728,266 & 11,173,962 & 1,514,841 & 4,386,178 \\
Optimizer & Adam & Adam & Adam & Adam & Adam \\
Learning Rate ($\eta$) & $1\text{e-}3$ & $1\text{e-}3$ & $1\text{e-}3$ & $1\text{e-}3$ & $2\text{e-}5$ \\
Weight Decay & $1\text{e-}4$ & $1\text{e-}4$ & $1\text{e-}4$ & $1\text{e-}4$ & $1\text{e-}2$ \\
Batch Size & 64 & 64 & 64 & 64 & 64 \\
Training Epochs & 20 & 20 & 20 & 20 & 20 \\
Loss Function & Cross-Entropy & Cross-Entropy & Cross-Entropy & CRF NLL Loss & Cross-Entropy \\
Prune Interval ($\Delta T$) & 500 steps & 500 steps & 500 steps & 500 steps & 500 steps \\
Threshold Sweep ($\tau$) & $[1\text{e-}6, 1\text{e-}4]$ & $[1\text{e-}6, 1\text{e-}4]$ & $[1\text{e-}6, 5\text{e-}4]$ & $[1\text{e-}6, 3\text{e-}4]$ & $[1\text{e-}6, 5\text{e-}6]$ \\
Precision & FP16 (AMP) & FP16 (AMP) & FP16 (AMP) & FP32 & FP16 (AMP) \\
Hardware Accelerator & NVIDIA T4 & NVIDIA T4 & NVIDIA T4 & NVIDIA T4 & NVIDIA T4 \\
\bottomrule
\end{tabular}%
}
\end{table}

\subsection{Baseline Pruning Implementation Details}
\label{sec:appendix_baselines}

To ensure a fair comparison, all models were trained for a total of 20 epochs under matching computational budgets. We adapted the baseline pruning methods to fit this schedule:

\begin{itemize}
    \item \textbf{Magnitude Pruning:} We train the model as a dense network for 17 epochs. At epoch 17, we prune the smallest weights to reach the target sparsity. We then fine-tune the sparse model for the remaining 3 epochs.
    \item \textbf{SNIP:} We calculate the pruning mask at epoch 0 before training begins using a single batch of data. Once the lowest-scoring weights are pruned, the mask stays fixed for all 20 training epochs.
    \item \textbf{RigL:} We start training from a random sparse mask at epoch 0. Every 100 steps, we drop small weights and regrow connections with large gradients. To keep training fast within our 20-epoch limit, we gradually slow down the regrowth rate over the first 16 epochs using a cosine curve, then freeze the mask for the final 4 epochs for fine-tuning.
\end{itemize}

\clearpage
\section{Full Empirical Benchmark Tables}
\label{sec:appendix_full_tables}

In this section, we provide complete benchmark evaluation tables across all models, datasets, pruning thresholds, target sparsities, and weight initialization strategies. DADP results report mean $\pm$ standard deviation across 3 random seeds ($42, 512, 1729$) from statistical sweeps.

\subsection{MLP on MNIST}

\begin{table}[H]
\caption{Complete empirical evaluation for MLP on MNIST across pruning methods and thresholds (DADP reports mean $\pm$ std across 3 random seeds).}
\label{tab:app_mlp_mnist}
\centering
\small
\resizebox{\linewidth}{!}{%
\begin{tabular}{llccc}
\toprule
\textbf{Method} & \textbf{Configuration / Threshold} & \textbf{Final Sparsity (\%)} & \textbf{Final Test Acc (\%)} & \textbf{Peak Test Acc (\%)} \\
\midrule
Dense Baseline & Unpruned (0\%) & 0.00\% & 98.39\% & 98.44\% \\
\midrule
\textbf{DADP (Ours)} & $\tau = 1\text{e-}6$ & \textbf{84.90 $\pm$ 0.56\%} & \textbf{97.91 $\pm$ 0.12\%} & \textbf{98.18 $\pm$ 0.02\%} \\
\textbf{DADP (Ours)} & $\tau = 5\text{e-}6$ & \textbf{94.60 $\pm$ 0.14\%} & \textbf{97.55 $\pm$ 0.09\%} & \textbf{98.02 $\pm$ 0.03\%} \\
\textbf{DADP (Ours)} & $\tau = 1\text{e-}5$ & \textbf{96.19 $\pm$ 0.11\%} & \textbf{97.38 $\pm$ 0.15\%} & \textbf{97.91 $\pm$ 0.08\%} \\
\textbf{DADP (Ours)} & $\tau = 5\text{e-}5$ & \textbf{98.35 $\pm$ 0.04\%} & \textbf{96.69 $\pm$ 0.04\%} & \textbf{96.92 $\pm$ 0.05\%} \\
\textbf{DADP (Ours)} & $\tau = 1\text{e-}5$ (100 Epoch Limit Test) & \textbf{98.00\%} & \textbf{96.94\%} & \textbf{97.69\%} \\
\textbf{DADP (Ours)} & $\tau = 1\text{e-}4$ (Over-threshold) & \textbf{99.81\%} & \textbf{77.11\%} & \textbf{77.30\%} \\
\midrule
Magnitude Pruning & Target 70.0\% & 70.00\% & 98.58\% & 98.58\% \\
Magnitude Pruning & Target 80.0\% & 80.00\% & 98.63\% & 98.65\% \\
Magnitude Pruning & Target 90.0\% & 90.00\% & 98.49\% & 98.49\% \\
Magnitude Pruning & Target 95.0\% & 95.00\% & 98.05\% & 98.38\% \\
\midrule
RigL & Target 70.0\% & 70.00\% & 98.18\% & 98.25\% \\
RigL & Target 80.0\% & 80.00\% & 97.97\% & 98.21\% \\
RigL & Target 90.0\% & 90.00\% & 97.81\% & 98.00\% \\
RigL & Target 95.0\% & 95.00\% & 97.65\% & 97.94\% \\
\midrule
SNIP & Target 70.0\% & 70.00\% & 98.23\% & 98.40\% \\
SNIP & Target 80.0\% & 80.00\% & 97.85\% & 98.29\% \\
SNIP & Target 90.0\% & 90.00\% & 97.89\% & 98.18\% \\
SNIP & Target 95.0\% & 95.00\% & 97.82\% & 97.94\% \\
\bottomrule
\end{tabular}%
}
\end{table}

\subsection{ResNet-18 on CIFAR-10}

\begin{table}[H]
\caption{Complete empirical evaluation for ResNet-18 on CIFAR-10 across pruning methods and thresholds (DADP reports mean $\pm$ std across 3 random seeds).}
\label{tab:app_resnet_cifar}
\centering
\small
\resizebox{\linewidth}{!}{%
\begin{tabular}{llccc}
\toprule
\textbf{Method} & \textbf{Configuration / Threshold} & \textbf{Final Sparsity (\%)} & \textbf{Final Test Acc (\%)} & \textbf{Peak Test Acc (\%)} \\
\midrule
Dense Baseline & Unpruned (0\%) & 0.00\% & 76.06\% & 77.44\% \\
\midrule
\textbf{DADP (Ours)} & $\tau = 1\text{e-}6$ & \textbf{84.65 $\pm$ 0.21\%} & \textbf{76.72 $\pm$ 0.26\%} & \textbf{76.97 $\pm$ 0.25\%} \\
\textbf{DADP (Ours)} & $\tau = 5\text{e-}6$ & \textbf{89.66 $\pm$ 0.10\%} & \textbf{76.64 $\pm$ 0.75\%} & \textbf{77.36 $\pm$ 0.17\%} \\
\textbf{DADP (Ours)} & $\tau = 1\text{e-}5$ & \textbf{91.41 $\pm$ 0.08\%} & \textbf{76.76 $\pm$ 0.16\%} & \textbf{77.25 $\pm$ 0.24\%} \\
\textbf{DADP (Ours)} & $\tau = 5\text{e-}5$ & \textbf{95.63 $\pm$ 0.05\%} & \textbf{76.15 $\pm$ 0.34\%} & \textbf{77.08 $\pm$ 0.16\%} \\
\textbf{DADP (Ours)} & $\tau = 1\text{e-}4$ & \textbf{96.92 $\pm$ 0.01\%} & \textbf{76.03 $\pm$ 0.61\%} & \textbf{76.80 $\pm$ 0.21\%} \\
\textbf{DADP (Ours)} & $\tau = 5\text{e-}4$ (High Sparsity) & \textbf{99.35 $\pm$ 0.03\%} & \textbf{72.28 $\pm$ 0.28\%} & \textbf{72.90 $\pm$ 0.24\%} \\
\midrule
Magnitude Pruning & Target 70.0\% & 70.00\% & 77.41\% & 77.41\% \\
Magnitude Pruning & Target 80.0\% & 80.00\% & 77.51\% & 77.53\% \\
Magnitude Pruning & Target 90.0\% & 90.00\% & 77.07\% & 77.32\% \\
Magnitude Pruning & Target 95.0\% & 95.00\% & 77.18\% & 77.18\% \\
Magnitude Pruning & Target 99.0\% & 99.00\% & 66.42\% & 77.21\% \\
\midrule
RigL & Target 70.0\% & 70.00\% & 74.14\% & 74.76\% \\
RigL & Target 80.0\% & 80.00\% & 72.96\% & 74.01\% \\
RigL & Target 90.0\% & 90.00\% & 71.53\% & 72.34\% \\
RigL & Target 95.0\% & 95.00\% & 69.60\% & 70.45\% \\
RigL & Target 99.0\% & 99.00\% & 63.85\% & 64.36\% \\
\midrule
SNIP & Target 70.0\% & 70.00\% & 75.99\% & 76.94\% \\
SNIP & Target 80.0\% & 80.00\% & 75.92\% & 76.83\% \\
SNIP & Target 90.0\% & 90.00\% & 75.21\% & 76.10\% \\
SNIP & Target 95.0\% & 95.00\% & 75.49\% & 75.49\% \\
SNIP & Target 99.0\% & 99.00\% & 71.44\% & 72.30\% \\
\bottomrule
\end{tabular}%
}
\end{table}

\subsection{BiLSTM-CRF on CoNLL-2003}

\begin{table}[H]
\caption{Complete empirical evaluation for BiLSTM-CRF on CoNLL-2003 (DADP reports mean $\pm$ std across 3 random seeds).}
\label{tab:app_bilstm_conll}
\centering
\small
\resizebox{\linewidth}{!}{%
\begin{tabular}{llccc}
\toprule
\textbf{Method} & \textbf{Configuration / Threshold} & \textbf{Final Sparsity (\%)} & \textbf{Final Test F1 (\%)} & \textbf{Peak Test F1 (\%)} \\
\midrule
Dense Baseline & Unpruned (0\%) & 0.00\% & 85.16\% & 85.16\% \\
\midrule
\textbf{DADP (Ours)} & $\tau = 1\text{e-}6$ & \textbf{12.10 $\pm$ 3.16\%} & \textbf{93.77 $\pm$ 0.12\%} & \textbf{93.94 $\pm$ 0.13\%} \\
\textbf{DADP (Ours)} & $\tau = 5\text{e-}6$ & \textbf{26.19 $\pm$ 3.80\%} & \textbf{93.82 $\pm$ 0.06\%} & \textbf{93.96 $\pm$ 0.15\%} \\
\textbf{DADP (Ours)} & $\tau = 1\text{e-}5$ & \textbf{33.57 $\pm$ 4.06\%} & \textbf{93.82 $\pm$ 0.14\%} & \textbf{93.95 $\pm$ 0.13\%} \\
\textbf{DADP (Ours)} & $\tau = 5\text{e-}5$ & \textbf{95.00 $\pm$ 3.60\%} & \textbf{84.28 $\pm$ 1.34\%} & \textbf{93.82 $\pm$ 0.23\%} \\
\textbf{DADP (Ours)} & $\tau = 1\text{e-}4$ & \textbf{95.09 $\pm$ 0.98\%} & \textbf{87.99 $\pm$ 1.73\%} & \textbf{93.25 $\pm$ 0.05\%} \\
\textbf{DADP (Ours)} & $\tau = 3\text{e-}4$ (Over-threshold) & \textbf{100.00\%} & \textbf{82.82\%} & \textbf{92.34\%} \\
\midrule
Magnitude Pruning & Target 70.0\% & 70.00\% & 93.87\% & 94.18\% \\
Magnitude Pruning & Target 80.0\% & 80.00\% & 93.88\% & 94.47\% \\
Magnitude Pruning & Target 90.0\% & 90.00\% & 94.06\% & 94.18\% \\
Magnitude Pruning & Target 95.0\% & 95.00\% & 93.88\% & 93.88\% \\
\midrule
RigL & Target 70.0\% & 70.00\% & 93.26\% & 94.06\% \\
RigL & Target 80.0\% & 80.00\% & 93.44\% & 94.00\% \\
RigL & Target 90.0\% & 90.00\% & 93.78\% & 94.32\% \\
RigL & Target 95.0\% & 95.00\% & 93.64\% & 93.78\% \\
\midrule
SNIP & Target 70.0\% & 70.00\% & 93.50\% & 94.28\% \\
SNIP & Target 80.0\% & 80.00\% & 93.37\% & 93.77\% \\
SNIP & Target 90.0\% & 90.00\% & 93.19\% & 94.24\% \\
SNIP & Target 95.0\% & 95.00\% & 92.60\% & 94.02\% \\
\bottomrule
\end{tabular}%
}
\end{table}

\subsection{MiniBERT on SST-2}

\begin{table}[H]
\caption{Complete empirical evaluation for MiniBERT on SST-2 across pruning methods and thresholds (DADP reports mean $\pm$ std across 3 random seeds).}
\label{tab:app_bert_sst2}
\centering
\small
\resizebox{\linewidth}{!}{%
\begin{tabular}{llccc}
\toprule
\textbf{Method} & \textbf{Configuration / Threshold} & \textbf{Final Sparsity (\%)} & \textbf{Final Test Acc (\%)} & \textbf{Peak Test Acc (\%)} \\
\midrule
Dense Baseline & Unpruned (0\%) & 0.00\% & 80.62\% & 82.68\% \\
\midrule
\textbf{DADP (Ours)} & $\tau = 1\text{e-}6$ & \textbf{71.52 $\pm$ 4.93\%} & \textbf{79.01 $\pm$ 0.57\%} & \textbf{82.19 $\pm$ 0.39\%} \\
\textbf{DADP (Ours)} & $\tau = 2\text{e-}6$ & \textbf{91.69 $\pm$ 2.03\%} & \textbf{78.67 $\pm$ 0.66\%} & \textbf{81.57 $\pm$ 0.30\%} \\
\textbf{DADP (Ours)} & $\tau = 3\text{e-}6$ & \textbf{96.03 $\pm$ 0.79\%} & \textbf{80.85 $\pm$ 1.06\%} & \textbf{82.11 $\pm$ 0.37\%} \\
\textbf{DADP (Ours)} & $\tau = 4\text{e-}6$ (Over-threshold) & \textbf{96.87 $\pm$ 0.16\%} & \textbf{64.18 $\pm$ 12.70\%} & \textbf{64.22 $\pm$ 12.76\%} \\
\textbf{DADP (Ours)} & $\tau = 5\text{e-}6$ (Over-threshold) & \textbf{96.96\%} & \textbf{50.92\%} & \textbf{50.92\%} \\
\midrule
Magnitude Pruning & Target 70.0\% & 70.00\% & 50.92\% & 82.57\% \\
Magnitude Pruning & Target 80.0\% & 80.00\% & 50.92\% & 82.34\% \\
Magnitude Pruning & Target 90.0\% & 90.00\% & 50.92\% & 83.03\% \\
Magnitude Pruning & Target 95.0\% & 95.00\% & 50.92\% & 82.34\% \\
\midrule
RigL & Target 70.0\% & 70.00\% & 80.62\% & 80.62\% \\
RigL & Target 80.0\% & 80.00\% & 80.73\% & 81.42\% \\
RigL & Target 90.0\% & 90.00\% & 80.62\% & 81.77\% \\
RigL & Target 95.0\% & 95.00\% & 81.31\% & 81.31\% \\
\midrule
SNIP & Target 70.0\% & 70.00\% & 81.54\% & 82.22\% \\
SNIP & Target 80.0\% & 80.00\% & 81.54\% & 82.00\% \\
SNIP & Target 90.0\% & 90.00\% & 82.68\% & 82.80\% \\
SNIP & Target 95.0\% & 95.00\% & 81.65\% & 82.22\% \\
\bottomrule
\end{tabular}%
}
\end{table}

\section{Complete VGG-16 Empirical Evaluation \& Visual Profiles}
\label{sec:appendix_vgg16}

To provide a comprehensive examination of DADP's behavior on large feedforward convolutional networks, this section consolidates all empirical benchmarks, layer-wise sparsity allocations, information-theoretic representation metrics, and weight distribution spectra for VGG-16 on CIFAR-10.

\subsection{VGG-16 Numerical Benchmarks}

\begin{table}[!ht]
\caption{Complete empirical evaluation for VGG-16 on CIFAR-10 across pruning methods and thresholds (DADP reports mean $\pm$ std across 3 random seeds).}
\label{tab:app_vgg_cifar}
\centering
\small
\resizebox{\linewidth}{!}{%
\begin{tabular}{llccc}
\toprule
\textbf{Method} & \textbf{Configuration / Threshold} & \textbf{Final Sparsity (\%)} & \textbf{Final Test Acc (\%)} & \textbf{Peak Test Acc (\%)} \\
\midrule
Dense Baseline & Unpruned (0\%) & 0.00\% & 85.21\% & 85.21\% \\
\midrule
\textbf{DADP (Ours)} & $\tau = 1\text{e-}6$ & \textbf{74.26 $\pm$ 1.25\%} & \textbf{85.09 $\pm$ 0.17\%} & \textbf{85.11 $\pm$ 0.16\%} \\
\textbf{DADP (Ours)} & $\tau = 5\text{e-}6$ & \textbf{89.77 $\pm$ 0.08\%} & \textbf{84.70 $\pm$ 0.17\%} & \textbf{85.47 $\pm$ 0.34\%} \\
\textbf{DADP (Ours)} & $\tau = 6\text{e-}6$ & \textbf{91.65 $\pm$ 0.67\%} & \textbf{82.65 $\pm$ 0.65\%} & \textbf{84.29 $\pm$ 0.53\%} \\
\textbf{DADP (Ours)} & $\tau = 6.5\text{e-}6$ & \textbf{91.86 $\pm$ 0.54\%} & \textbf{83.22 $\pm$ 0.00\%} & \textbf{83.22 $\pm$ 0.00\%} \\
\textbf{DADP (Ours)} & $\tau = 6.7\text{e-}6$ & \textbf{92.16 $\pm$ 0.28\%} & \textbf{83.06 $\pm$ 2.01\%} & \textbf{84.02 $\pm$ 1.36\%} \\
\textbf{DADP (Ours)} & $\tau = 1\text{e-}5$ (Over-threshold) & \textbf{96.50\%} & \textbf{10.00\%} & \textbf{18.84\%} \\
\textbf{DADP (Ours)} & $\tau = 1\text{e-}4$ (Over-threshold) & \textbf{100.00\%} & \textbf{10.00\%} & \textbf{10.00\%} \\
\midrule
Magnitude Pruning & Target 70.0\% & 70.00\% & 86.03\% & 86.22\% \\
Magnitude Pruning & Target 80.0\% & 80.00\% & 85.75\% & 85.75\% \\
Magnitude Pruning & Target 90.0\% & 90.00\% & 85.56\% & 85.56\% \\
Magnitude Pruning & Target 95.0\% & 95.00\% & 82.83\% & 84.49\% \\
\midrule
RigL & Target 70.0\% & 70.00\% & 83.51\% & 84.07\% \\
RigL & Target 80.0\% & 80.00\% & 81.20\% & 83.67\% \\
RigL & Target 90.0\% & 90.00\% & 82.38\% & 82.38\% \\
RigL & Target 95.0\% & 95.00\% & 80.23\% & 80.58\% \\
\midrule
SNIP & Target 70.0\% & 70.00\% & 85.39\% & 85.90\% \\
SNIP & Target 80.0\% & 80.00\% & 84.83\% & 85.71\% \\
SNIP & Target 90.0\% & 90.00\% & 84.65\% & 85.11\% \\
SNIP & Target 95.0\% & 95.03\% & 14.06\% & 23.44\% \\
\bottomrule
\end{tabular}%
}
\end{table}

\subsection{VGG-16 Visual Profiles \& Representations}

\begin{figure}[H]
\centering
\begin{subfigure}[b]{\linewidth}
    \centering
    \includegraphics[width=\linewidth]{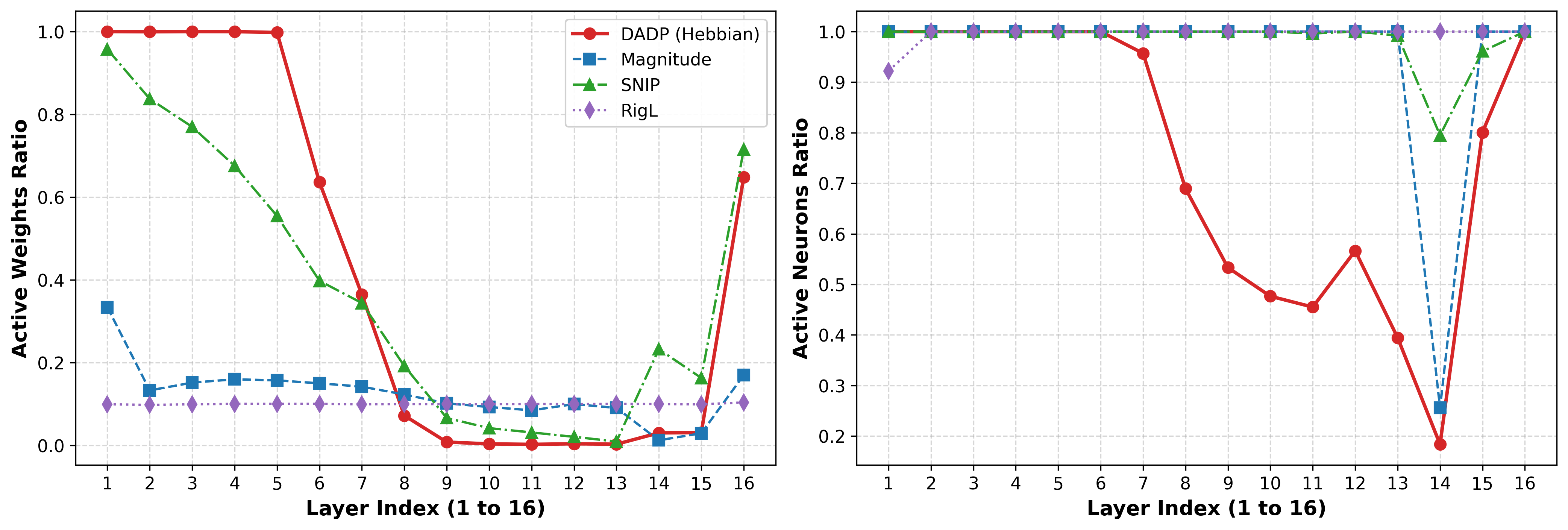}
    \caption{Layer-Wise Sparsity Profile ($90\%$ Sparsity)}
    \label{fig:app_vgg_layer_wise}
\end{subfigure}
\vspace{0.3cm}
\begin{subfigure}[b]{\linewidth}
    \centering
    \includegraphics[width=\linewidth]{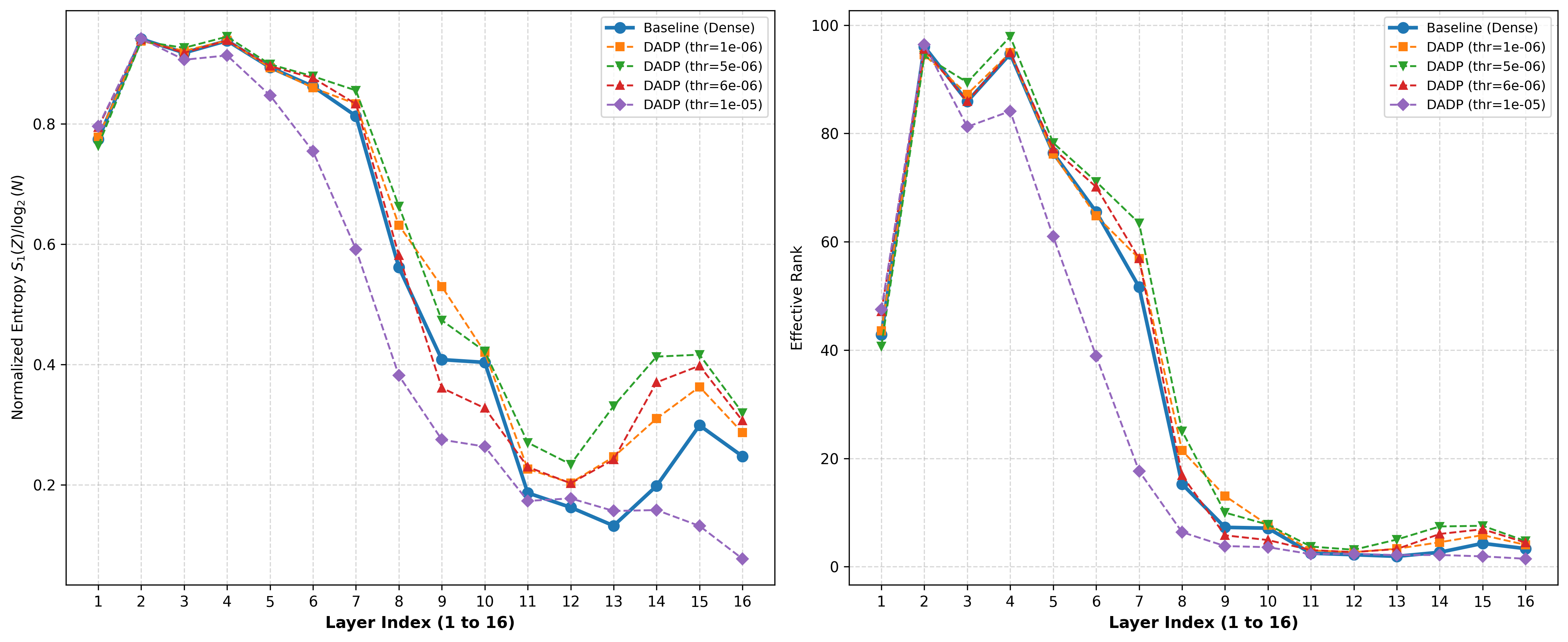}
    \caption{Information-Theoretic Representation Quality (Matrix Entropy \& Effective Rank)}
    \label{fig:app_vgg_entropy}
\end{subfigure}
\vspace{0.3cm}
\begin{subfigure}[b]{\linewidth}
    \centering
    \includegraphics[width=\linewidth]{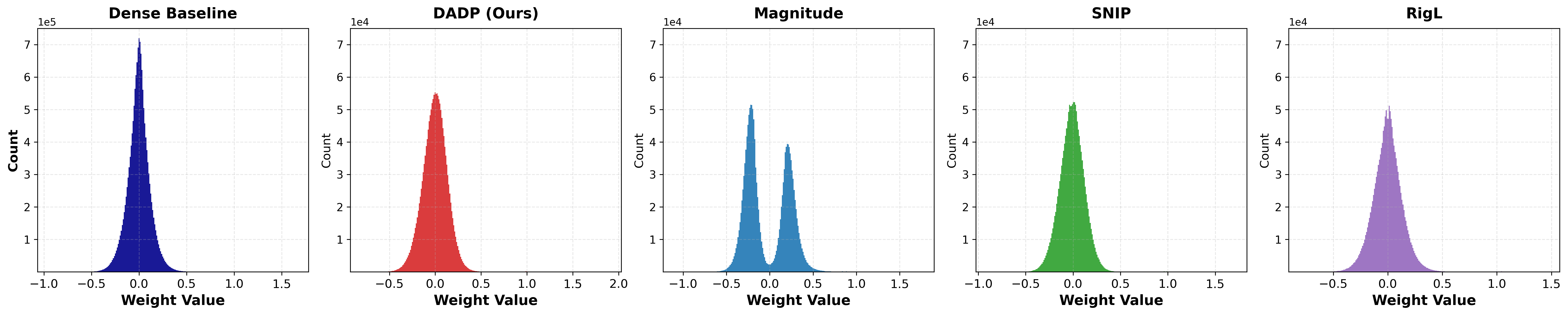}
    \caption{Global Active Weight Value Distributions across Pruning Methods}
    \label{fig:app_vgg_weight_dist}
\end{subfigure}
\caption{Comprehensive visual evaluation of VGG-16 on CIFAR-10. (a) DADP preserves classification accuracy up to $92.16\%$ global sparsity. (b) Layer-wise profile shows early spatial filter lock alongside heavy compression in deep linear projection layers (\texttt{classifier.0}). (c) Latent representation quality across all 16 layers confirms that DADP maintains normalized Shannon entropy ($S_{\text{norm}} \ge 0.85$) on par with dense unpruned features. (d) DADP preserves a smooth, zero-centered parameter distribution, avoiding artificial bimodal gaps around zero.}
\label{fig:app_vgg_comprehensive}
\end{figure}

\clearpage
\section{Detailed Channel-Level Structured Compression Tables}
\label{sec:appendix_structured_compression}

In this section, we present complete layer-by-layer physical structured compression shapes for VGG-16 (at $\sim 90\%$ global sparsity) and ResNet-18 (at $\sim 95\%$ global sparsity) across all evaluated pruning methods (DADP, SNIP, Magnitude, and RigL). Physical structured pruning deletes completely dead channels and neurons (where $100\%$ of constituent weights are inactive), compressing tensor dimensions directly on disk.

\subsection{VGG-16 on CIFAR-10 ($\sim 90\%$ Global Sparsity)}

\begin{table}[H]
\caption{Layer-by-layer physical structured compression shapes for VGG-16 on CIFAR-10 at $\sim 90\%$ global sparsity across pruning methods, with the number of pruned (dead) channels/neurons in parentheses.}
\label{tab:app_vgg_structured_compression}
\centering
\small
\resizebox{\linewidth}{!}{%
\begin{tabular}{lccccc}
\toprule
\textbf{Layer Name} & \textbf{Original Shape} & \textbf{DADP Compressed Shape} & \textbf{SNIP Compressed Shape} & \textbf{Magnitude Compressed Shape} & \textbf{RigL Compressed Shape} \\
\midrule
\texttt{features.0}  & $[64, 3, 3, 3]$   & \textbf{$[64, 3, 3, 3]$} {\scriptsize(0 dead)}     & $[64, 3, 3, 3]$ {\scriptsize(0 dead)}     & $[64, 3, 3, 3]$ {\scriptsize(0 dead)}     & $[59, 3, 3, 3]$ {\scriptsize(5 dead)} \\
\texttt{features.3}  & $[64, 64, 3, 3]$  & \textbf{$[64, 64, 3, 3]$} {\scriptsize(0 dead)}    & $[64, 64, 3, 3]$ {\scriptsize(0 dead)}    & $[64, 64, 3, 3]$ {\scriptsize(0 dead)}    & $[64, 59, 3, 3]$ {\scriptsize(0 dead)} \\
\texttt{features.7}  & $[128, 64, 3, 3]$ & \textbf{$[128, 64, 3, 3]$} {\scriptsize(0 dead)}   & $[128, 64, 3, 3]$ {\scriptsize(0 dead)}   & $[128, 64, 3, 3]$ {\scriptsize(0 dead)}   & $[128, 64, 3, 3]$ {\scriptsize(0 dead)} \\
\texttt{features.10} & $[128, 128, 3, 3]$& \textbf{$[128, 128, 3, 3]$} {\scriptsize(0 dead)}  & $[128, 128, 3, 3]$ {\scriptsize(0 dead)}  & $[128, 128, 3, 3]$ {\scriptsize(0 dead)}  & $[128, 128, 3, 3]$ {\scriptsize(0 dead)} \\
\texttt{features.14} & $[256, 128, 3, 3]$& \textbf{$[256, 128, 3, 3]$} {\scriptsize(0 dead)}  & $[256, 128, 3, 3]$ {\scriptsize(0 dead)}  & $[256, 128, 3, 3]$ {\scriptsize(0 dead)}  & $[256, 128, 3, 3]$ {\scriptsize(0 dead)} \\
\texttt{features.17} & $[256, 256, 3, 3]$& \textbf{$[256, 256, 3, 3]$} {\scriptsize(0 dead)}  & $[256, 256, 3, 3]$ {\scriptsize(0 dead)}  & $[256, 256, 3, 3]$ {\scriptsize(0 dead)}  & $[256, 256, 3, 3]$ {\scriptsize(0 dead)} \\
\texttt{features.20} & $[256, 256, 3, 3]$& \textbf{$[256, 256, 3, 3]$} {\scriptsize(0 dead)}  & $[256, 256, 3, 3]$ {\scriptsize(0 dead)}  & $[256, 256, 3, 3]$ {\scriptsize(0 dead)}  & $[256, 256, 3, 3]$ {\scriptsize(0 dead)} \\
\texttt{features.24} & $[512, 256, 3, 3]$& \textbf{$[498, 256, 3, 3]$} {\scriptsize(14 dead)} & $[512, 256, 3, 3]$ {\scriptsize(0 dead)}  & $[512, 256, 3, 3]$ {\scriptsize(0 dead)}  & $[512, 256, 3, 3]$ {\scriptsize(0 dead)} \\
\texttt{features.27} & $[512, 512, 3, 3]$& \textbf{$[444, 498, 3, 3]$} {\scriptsize(68 dead)} & $[512, 512, 3, 3]$ {\scriptsize(0 dead)}  & $[512, 512, 3, 3]$ {\scriptsize(0 dead)}  & $[512, 512, 3, 3]$ {\scriptsize(0 dead)} \\
\texttt{features.30} & $[512, 512, 3, 3]$& \textbf{$[399, 444, 3, 3]$} {\scriptsize(113 dead)}& $[512, 512, 3, 3]$ {\scriptsize(0 dead)}  & $[512, 512, 3, 3]$ {\scriptsize(0 dead)}  & $[512, 512, 3, 3]$ {\scriptsize(0 dead)} \\
\texttt{features.34} & $[512, 512, 3, 3]$& \textbf{$[391, 399, 3, 3]$} {\scriptsize(121 dead)}& $[510, 512, 3, 3]$ {\scriptsize(2 dead)}   & $[512, 512, 3, 3]$ {\scriptsize(0 dead)}  & $[512, 512, 3, 3]$ {\scriptsize(0 dead)} \\
\texttt{features.37} & $[512, 512, 3, 3]$& \textbf{$[338, 391, 3, 3]$} {\scriptsize(174 dead)}& $[512, 510, 3, 3]$ {\scriptsize(0 dead)}   & $[512, 512, 3, 3]$ {\scriptsize(0 dead)}  & $[512, 512, 3, 3]$ {\scriptsize(0 dead)} \\
\texttt{features.40} & $[512, 512, 3, 3]$& \textbf{$[293, 338, 3, 3]$} {\scriptsize(219 dead)}& $[508, 512, 3, 3]$ {\scriptsize(4 dead)}   & $[512, 512, 3, 3]$ {\scriptsize(0 dead)}  & $[512, 512, 3, 3]$ {\scriptsize(0 dead)} \\
\texttt{classifier.0}& $[512, 512]$      & \textbf{$[103, 293]$} {\scriptsize(409 dead)}      & $[407, 508]$ {\scriptsize(105 dead)}       & $[131, 512]$ {\scriptsize(381 dead)}       & $[512, 512]$ {\scriptsize(0 dead)} \\
\texttt{classifier.3}& $[512, 512]$      & \textbf{$[452, 103]$} {\scriptsize(60 dead)}       & $[492, 407]$ {\scriptsize(20 dead)}        & $[512, 131]$ {\scriptsize(0 dead)}        & $[512, 512]$ {\scriptsize(0 dead)} \\
\texttt{classifier.6}& $[10, 512]$       & \textbf{$[10, 452]$} {\scriptsize(0 dead)}         & $[10, 492]$ {\scriptsize(0 dead)}          & $[10, 512]$ {\scriptsize(0 dead)}          & $[10, 512]$ {\scriptsize(0 dead)} \\
\bottomrule
\end{tabular}%
}
\end{table}

\subsection{ResNet-18 on CIFAR-10 ($\sim 95\%$ Global Sparsity)}

\begin{table}[H]
\caption{Layer-by-layer physical structured compression shapes for ResNet-18 on CIFAR-10 at $\sim 95\%$ global sparsity across pruning methods, with the number of pruned (dead) channels/neurons in parentheses.}
\label{tab:app_resnet_structured_compression}
\centering
\small
\resizebox{\linewidth}{!}{%
\begin{tabular}{lccccc}
\toprule
\textbf{Layer Name} & \textbf{Original Shape} & \textbf{DADP Compressed Shape} & \textbf{SNIP Compressed Shape} & \textbf{Magnitude Compressed Shape} & \textbf{RigL Compressed Shape} \\
\midrule
\texttt{conv1}                 & $[64, 3, 3, 3]$    & \textbf{$[64, 3, 3, 3]$} {\scriptsize(0 dead)}     & $[64, 3, 3, 3]$ {\scriptsize(0 dead)}    & $[64, 3, 3, 3]$ {\scriptsize(0 dead)}    & $[64, 3, 3, 3]$ {\scriptsize(0 dead)} \\
\texttt{layer1.0.conv1}        & $[64, 64, 3, 3]$   & \textbf{$[64, 64, 3, 3]$} {\scriptsize(0 dead)}    & $[64, 64, 3, 3]$ {\scriptsize(0 dead)}   & $[64, 64, 3, 3]$ {\scriptsize(0 dead)}   & $[64, 64, 3, 3]$ {\scriptsize(0 dead)} \\
\texttt{layer1.0.conv2}        & $[64, 64, 3, 3]$   & \textbf{$[64, 64, 3, 3]$} {\scriptsize(0 dead)}    & $[64, 64, 3, 3]$ {\scriptsize(0 dead)}   & $[64, 64, 3, 3]$ {\scriptsize(0 dead)}   & $[64, 64, 3, 3]$ {\scriptsize(0 dead)} \\
\texttt{layer1.1.conv1}        & $[64, 64, 3, 3]$   & \textbf{$[64, 64, 3, 3]$} {\scriptsize(0 dead)}    & $[64, 64, 3, 3]$ {\scriptsize(0 dead)}   & $[64, 64, 3, 3]$ {\scriptsize(0 dead)}   & $[64, 64, 3, 3]$ {\scriptsize(0 dead)} \\
\texttt{layer1.1.conv2}        & $[64, 64, 3, 3]$   & \textbf{$[64, 64, 3, 3]$} {\scriptsize(0 dead)}    & $[64, 64, 3, 3]$ {\scriptsize(0 dead)}   & $[64, 64, 3, 3]$ {\scriptsize(0 dead)}   & $[64, 64, 3, 3]$ {\scriptsize(0 dead)} \\
\texttt{layer2.0.conv1}        & $[128, 64, 3, 3]$  & \textbf{$[128, 64, 3, 3]$} {\scriptsize(0 dead)}   & $[128, 64, 3, 3]$ {\scriptsize(0 dead)}  & $[128, 64, 3, 3]$ {\scriptsize(0 dead)}  & $[128, 64, 3, 3]$ {\scriptsize(0 dead)} \\
\texttt{layer2.0.conv2}        & $[128, 128, 3, 3]$ & \textbf{$[128, 128, 3, 3]$} {\scriptsize(0 dead)}  & $[128, 128, 3, 3]$ {\scriptsize(0 dead)} & $[128, 128, 3, 3]$ {\scriptsize(0 dead)} & $[128, 128, 3, 3]$ {\scriptsize(0 dead)} \\
\texttt{layer2.0.downsample.0} & $[128, 64, 1, 1]$  & \textbf{$[128, 64, 1, 1]$} {\scriptsize(0 dead)}   & $[128, 64, 1, 1]$ {\scriptsize(0 dead)}  & $[128, 64, 1, 1]$ {\scriptsize(0 dead)}  & $[128, 64, 1, 1]$ {\scriptsize(0 dead)} \\
\texttt{layer2.1.conv1}        & $[128, 128, 3, 3]$ & \textbf{$[128, 128, 3, 3]$} {\scriptsize(0 dead)}  & $[128, 128, 3, 3]$ {\scriptsize(0 dead)} & $[128, 128, 3, 3]$ {\scriptsize(0 dead)} & $[128, 128, 3, 3]$ {\scriptsize(0 dead)} \\
\texttt{layer2.1.conv2}        & $[128, 128, 3, 3]$ & \textbf{$[128, 128, 3, 3]$} {\scriptsize(0 dead)}  & $[128, 128, 3, 3]$ {\scriptsize(0 dead)} & $[128, 128, 3, 3]$ {\scriptsize(0 dead)} & $[128, 128, 3, 3]$ {\scriptsize(0 dead)} \\
\texttt{layer3.0.conv1}        & $[256, 128, 3, 3]$ & \textbf{$[191, 128, 3, 3]$} {\scriptsize(65 dead)}  & $[256, 128, 3, 3]$ {\scriptsize(0 dead)} & $[256, 128, 3, 3]$ {\scriptsize(0 dead)} & $[256, 128, 3, 3]$ {\scriptsize(0 dead)} \\
\texttt{layer3.0.conv2}        & $[256, 256, 3, 3]$ & \textbf{$[164, 191, 3, 3]$} {\scriptsize(92 dead)}  & $[256, 256, 3, 3]$ {\scriptsize(0 dead)} & $[256, 256, 3, 3]$ {\scriptsize(0 dead)} & $[256, 256, 3, 3]$ {\scriptsize(0 dead)} \\
\texttt{layer3.0.downsample.0} & $[256, 128, 1, 1]$ & \textbf{$[159, 128, 1, 1]$} {\scriptsize(97 dead)}  & $[256, 128, 1, 1]$ {\scriptsize(0 dead)} & $[256, 128, 1, 1]$ {\scriptsize(0 dead)} & $[256, 128, 1, 1]$ {\scriptsize(0 dead)} \\
\texttt{layer3.1.conv1}        & $[256, 256, 3, 3]$ & \textbf{$[101, 164, 3, 3]$} {\scriptsize(155 dead)} & $[256, 256, 3, 3]$ {\scriptsize(0 dead)} & $[256, 256, 3, 3]$ {\scriptsize(0 dead)} & $[256, 256, 3, 3]$ {\scriptsize(0 dead)} \\
\texttt{layer3.1.conv2}        & $[256, 256, 3, 3]$ & \textbf{$[142, 101, 3, 3]$} {\scriptsize(114 dead)} & $[256, 256, 3, 3]$ {\scriptsize(0 dead)} & $[256, 256, 3, 3]$ {\scriptsize(0 dead)} & $[256, 256, 3, 3]$ {\scriptsize(0 dead)} \\
\texttt{layer4.0.conv1}        & $[512, 256, 3, 3]$ & \textbf{$[138, 142, 3, 3]$} {\scriptsize(374 dead)} & $[512, 256, 3, 3]$ {\scriptsize(0 dead)} & $[512, 256, 3, 3]$ {\scriptsize(0 dead)} & $[512, 256, 3, 3]$ {\scriptsize(0 dead)} \\
\texttt{layer4.0.conv2}        & $[512, 512, 3, 3]$ & \textbf{$[256, 138, 3, 3]$} {\scriptsize(256 dead)} & $[512, 512, 3, 3]$ {\scriptsize(0 dead)} & $[512, 512, 3, 3]$ {\scriptsize(0 dead)} & $[512, 512, 3, 3]$ {\scriptsize(0 dead)} \\
\texttt{layer4.0.downsample.0} & $[512, 256, 1, 1]$ & \textbf{$[400, 142, 1, 1]$} {\scriptsize(112 dead)} & $[512, 256, 1, 1]$ {\scriptsize(0 dead)} & $[512, 256, 1, 1]$ {\scriptsize(0 dead)} & $[512, 256, 1, 1]$ {\scriptsize(0 dead)} \\
\texttt{layer4.1.conv1}        & $[512, 512, 3, 3]$ & \textbf{$[69, 256, 3, 3]$}  {\scriptsize(443 dead)} & $[512, 512, 3, 3]$ {\scriptsize(0 dead)} & $[505, 512, 3, 3]$ {\scriptsize(7 dead)}  & $[512, 512, 3, 3]$ {\scriptsize(0 dead)} \\
\texttt{layer4.1.conv2}        & $[512, 512, 3, 3]$ & \textbf{$[145, 69, 3, 3]$}  {\scriptsize(367 dead)} & $[505, 512, 3, 3]$ {\scriptsize(7 dead)} & $[495, 505, 3, 3]$ {\scriptsize(17 dead)} & $[512, 512, 3, 3]$ {\scriptsize(0 dead)} \\
\texttt{fc}                    & $[10, 512]$        & \textbf{$[10, 145]$}        {\scriptsize(0 dead)}   & $[10, 505]$        {\scriptsize(0 dead)} & $[10, 495]$        {\scriptsize(0 dead)} & $[10, 512]$        {\scriptsize(0 dead)} \\
\bottomrule
\end{tabular}%
}
\end{table}

\clearpage
\section{Initialization Sensitivity \& Architectural Layer Mapping}
\label{sec:appendix_sensitivity_mapping}

\subsection{Initialization Sensitivity Ablations}

\begin{table}[H]
\caption{Weight initialization scheme ablation evaluation for VGG-16 and ResNet-18 on CIFAR-10.}
\label{tab:app_init_ablation}
\centering
\small
\resizebox{\linewidth}{!}{%
\begin{tabular}{llccc}
\toprule
\textbf{Model} & \textbf{Initialization Scheme} & \textbf{Pruning Threshold $\tau$} & \textbf{Final Sparsity (\%)} & \textbf{Final Test Acc (\%)} \\
\midrule
ResNet-18 & Kaiming Normal & $\tau = 5\text{e-}5$ & 95.41\% & 75.10\% \\
ResNet-18 & Kaiming Uniform & $\tau = 5\text{e-}5$ & 95.35\% & 76.08\% \\
ResNet-18 & Xavier Normal & $\tau = 5\text{e-}5$ & 95.66\% & 77.05\% \\
ResNet-18 & Xavier Uniform & $\tau = 5\text{e-}5$ & 95.54\% & 76.72\% \\
ResNet-18 & Orthogonal & $\tau = 5\text{e-}5$ & 95.75\% & 76.51\% \\
ResNet-18 & Normal ($\sigma = 0.02$) & $\tau = 5\text{e-}5$ & 95.96\% & 77.26\% \\
ResNet-18 & Normal ($\sigma = 0.1$) & $\tau = 5\text{e-}5$ & 94.67\% & 74.42\% \\
\midrule
VGG-16 & Kaiming Normal & $\tau = 5\text{e-}6$ & 84.86\% & 84.80\% \\
VGG-16 & Kaiming Uniform & $\tau = 5\text{e-}6$ & 85.03\% & 85.30\% \\
VGG-16 & Xavier Normal & $\tau = 5\text{e-}6$ & 86.58\% & 84.84\% \\
VGG-16 & Xavier Uniform & $\tau = 5\text{e-}6$ & 86.73\% & 85.38\% \\
VGG-16 & Orthogonal & $\tau = 5\text{e-}6$ & 86.61\% & 85.63\% \\
VGG-16 & Normal ($\sigma = 0.02$) & $\tau = 5\text{e-}6$ & \textbf{100.00\%} & \textbf{10.00\%} \\
VGG-16 & Normal ($\sigma = 0.1$) & $\tau = 5\text{e-}6$ & 77.68\% & 83.44\% \\
\bottomrule
\end{tabular}%
}
\end{table}

\subsection{Layer Index to Architectural Name Mapping}
\label{sec:appendix_layer_mapping}

\begin{table}[H]
\caption{Sequential layer index mapping ($1, 2, 3, \dots, L$) to full parameter layer names for ResNet-18, VGG-16, and MiniBERT.}
\label{tab:app_layer_name_mapping}
\centering
\small
\resizebox{\linewidth}{!}{%
\begin{tabular}{c l l l}
\toprule
\textbf{Layer Index} & \textbf{ResNet-18 (CIFAR-10)} & \textbf{VGG-16 (CIFAR-10)} & \textbf{MiniBERT (SST-2)} \\
\midrule
1 & \texttt{conv1} & \texttt{features.0 (conv1\_1)} & \texttt{embeddings.word\_embeddings} \\
2 & \texttt{layer1.0.conv1} & \texttt{features.3 (conv1\_2)} & \texttt{embeddings.position\_embeddings} \\
3 & \texttt{layer1.0.conv2} & \texttt{features.7 (conv2\_1)} & \texttt{encoder.layer.0.attention.query} \\
4 & \texttt{layer1.1.conv1} & \texttt{features.10 (conv2\_2)} & \texttt{encoder.layer.0.attention.key} \\
5 & \texttt{layer1.1.conv2} & \texttt{features.14 (conv3\_1)} & \texttt{encoder.layer.0.attention.value} \\
6 & \texttt{layer2.0.conv1} & \texttt{features.17 (conv3\_2)} & \texttt{encoder.layer.0.output.dense} \\
7 & \texttt{layer2.0.conv2} & \texttt{features.20 (conv3\_3)} & \texttt{encoder.layer.0.intermediate.dense} \\
8 & \texttt{layer2.0.downsample.0} & \texttt{features.24 (conv4\_1)} & \texttt{encoder.layer.1.attention.query} \\
9 & \texttt{layer2.1.conv1} & \texttt{features.27 (conv4\_2)} & \texttt{encoder.layer.1.attention.key} \\
10 & \texttt{layer2.1.conv2} & \texttt{features.30 (conv4\_3)} & \texttt{encoder.layer.1.attention.value} \\
11 & \texttt{layer3.0.conv1} & \texttt{features.34 (conv5\_1)} & \texttt{encoder.layer.1.output.dense} \\
12 & \texttt{layer3.0.conv2} & \texttt{features.37 (conv5\_2)} & \texttt{encoder.layer.1.intermediate.dense} \\
13 & \texttt{layer3.0.downsample.0} & \texttt{features.40 (conv5\_3)} & \texttt{classifier} \\
14 & \texttt{layer3.1.conv1} & \texttt{classifier.0 (fc1)} & -- \\
15 & \texttt{layer3.1.conv2} & \texttt{classifier.3 (fc2)} & -- \\
16 & \texttt{layer4.0.conv1} & \texttt{classifier.6 (fc3)} & -- \\
17 & \texttt{layer4.0.conv2} & -- & -- \\
18 & \texttt{layer4.0.downsample.0} & -- & -- \\
19 & \texttt{layer4.1.conv1} & -- & -- \\
20 & \texttt{layer4.1.conv2} & -- & -- \\
21 & \texttt{fc} & -- & -- \\
\bottomrule
\end{tabular}%
}
\end{table}

\end{document}